\documentclass[manuscript,screen]{acmart}
\AtBeginDocument{%
  }

\setcopyright{acmlicensed}
\copyrightyear{2026}
\acmYear{2026}
\acmJournal{TKDD}
\makeatletter
\def\@journalName{ACM Transactions on Knowledge Discovery from Data}%
\def\@journalNameShort{ACM Trans. Knowl. Discov. Data}%
\makeatother
\acmDOI{}

\usepackage[nameinlink,capitalize,noabbrev]{cleveref}
\usepackage{mathtools}  
\crefname{figure}{Fig.}{Figs.}
\Crefname{figure}{Fig.}{Figs.}

\crefname{table}{Tab.}{Tabs.}
\Crefname{table}{Tab.}{Tabs.}

\crefname{proposition}{Prop.}{Props.}
\Crefname{proposition}{Prop.}{Props.}

\crefname{section}{Sec.}{Secs.}
\Crefname{section}{Sec.}{Secs.}

\crefname{subsection}{Sec.}{Secs.}
\Crefname{subsection}{Sec.}{Secs.}

\crefname{equation}{Eq.}{Eqs.}
\Crefname{equation}{Eq.}{Eqs.}

\crefname{theorem}{Thm.}{Thms.}
\Crefname{theorem}{Thm.}{Thms.}

\crefname{lemma}{Lem.}{Lems.}
\Crefname{lemma}{Lem.}{Lems.}

\crefname{appendix}{App.}{Apps.}
\Crefname{appendix}{App.}{Apps.}

\crefformat{figure}{#2Fig.~#1#3}
\Crefformat{figure}{#2Fig.~#1#3}
\crefformat{table}{#2Tab.~#1#3}
\Crefformat{table}{#2Tab.~#1#3}
\crefformat{section}{#2Sec.~#1#3}
\Crefformat{section}{#2Sec.~#1#3}

\usepackage{wrapfig}
\usepackage{todonotes}
\usepackage{hyperref}
\usepackage{url}
\usepackage{tikz}
\usetikzlibrary{graphs}
\usepackage[T1]{fontenc}    
\usepackage{url}            
\usepackage{booktabs}       
\usepackage{nicefrac}       
\usepackage{microtype}      
\usepackage{multirow}
\usepackage[normalem]{ulem}
\usepackage{breqn} 

\usepackage{ulem}
\usepackage{graphicx} 
\usepackage{outlines} 
\usepackage{nomencl}
\usepackage{comment}
\usepackage{booktabs}
\usepackage{natbib}
\usepackage{caption}
\usepackage{subcaption}
\usepackage{float}
\usepackage{comment}
\usepackage{wrapfig}
\usepackage{enumitem}
\usepackage{makecell} 
\usepackage{wrapfig}  
\usepackage[T1]{fontenc}    %
\usepackage{url}            %
\usepackage{booktabs}       %
\usepackage{nicefrac}       %
\usepackage{microtype}      %
\usepackage{tablefootnote}
\usepackage{etoolbox} %
\usepackage{graphicx}
\usepackage{subcaption}
\usepackage{listings}
\usepackage{float}
\usepackage{enumitem}
\usepackage{xspace}
\usepackage{multirow}
\usepackage{array}
\usepackage{tabularx}
\usepackage{booktabs}
\usepackage{xcolor,colortbl}
\usepackage{pifont} 
\usepackage{soul}
\usepackage{xfrac}
\usepackage{wrapfig}
\usepackage{mdwlist}
\usepackage{arydshln}
\usepackage{enumitem}
\usepackage{adjustbox}

\usepackage[ruled,vlined]{algorithm2e}

\definecolor{Apricot}{rgb}{0.98, 0.81, 0.69}
\usepackage{algorithm2e}   
\usepackage{hyperref}      
\usepackage{cleveref}      

\crefname{algocf}{Algorithm}{Algorithms}
\Crefname{algocf}{Algorithm}{Algorithms}

\newcommand{\cmark}{{\color{darkgreen}\ding{51}\xspace}}

\definecolor{darkgreen}{RGB}{0,153,0}
\definecolor{royalgreen}{RGB}{11,102,35}

\definecolor{blue}{RGB}{0,0,0}
\definecolor{royalgreen}{RGB}{0,0,0}
\definecolor{red}{RGB}{0,0,0}

\newcommand{\bymark}{{\color{black}\ding{51}\xspace}}  

\usepackage[T1]{fontenc}
\usepackage{graphicx}
\usepackage{booktabs}
\usepackage[misc]{ifsym}

\usepackage{hyperref}
\usepackage{enumitem}
\usepackage{mwe}
\usepackage{comment}
\usepackage[table]{xcolor}
\usepackage{adjustbox}
\usepackage{cite} 
\usepackage{graphicx} 
\usepackage[dvipsnames]{xcolor}

\usepackage[normalem]{ulem} 
\usepackage{multirow}
\usepackage{booktabs}    
\usepackage{pifont}      
\usepackage{subcaption}
\usepackage{wrapfig}
\usepackage{caption}
\usepackage{amsmath}
\usepackage{wrapfig}   
\usepackage{booktabs}  
\usepackage{graphicx}

\usepackage[dvipsnames]{xcolor}

\newcommand{\method}{\textsc{Bridge-LP}\xspace}

\newcommand{\mechanismOne}{Jumping Connection of Textual Proximity}
\newcommand{\mechanismTwo}{Compressed Embedding}
\newcommand{\mechanismThree}{Cache Embedding Training Strategy}

\newcommand{\signaltwo}{Semantic proximity\xspace}
\newcommand{\signalone}{Structural proximity\xspace}
\newcommand{\signalthree}{Propagated proximity\xspace}

\newcommand{\modify}[1]{{\textsf{\textcolor{red}{#1}}}}

\newcommand{\mechanismone}{jumping connection of textual proximity}
\newcommand{\mechanismtwo}{compressed embedding}
\newcommand{\mechanismthree}{cache embedding training strategy}

\newcommand{\Cone}{Common Neighbor Embedding }
\newcommand{\Ctwo}{Semantic Feature Proximity }
\newcommand{\Cthree}{Aggregated Proximity Feature }
\crefname{algorithm}{Algorithm}{Algorithms}

\begin{document}
\title[\method:  A Theoretically Grounded Framework for Link Prediction on TAGs]{Bridging Semantics and Structure: \\ A Theoretically Grounded Framework for Link Prediction on TAGs}

\author{Chen Shao}
\email{chen.shao2@kit.edu}
\orcid{0000-0002-7807-1090}
\affiliation{%
  \institution{Karlsruhe Institute of Technology (KIT)}
  \city{Karlsruhe}
  \country{Germany}
}
\author{Ruirui Wang}
\email{ruirui.wang@uzh.ch}
\orcid{0009-0001-4378-4920}
\affiliation{%
  \institution{University of Zurich}
  \city{Zurich}
  \state{Zurich}
  \country{Switzerland}
}
\author{Tobias Käfer}
\email{tobias.kaefer@kit.edu}
\orcid{0000-0003-0576-7457}
\affiliation{%
  \institution{Karlsruhe Institute of Technology (KIT)}
  \city{Karlsruhe}
  \country{Germany}
}
\author{Danai Koutra}
\email{dkoutra@umich.edu}
\orcid{0000-0002-3206-8179}
\affiliation{%
  \institution{University of Michigan}
  \city{Ann Arbor}
  \state{Michigan}
  \country{USA}
}

\renewcommand{\shortauthors}{Chen et al.}


\begin{abstract}
Link prediction on text-attributed graphs (TAGs) is a fundamental problem with applications in knowledge discovery, recommendation and beyond. Recent progress has highlighted the promise of combining Graph Neural Networks (GNNs) with Language Models (LMs) in order to leverage the complementary structural and semantic signals. However, due to the distinct inductive biases of GNNs and LMs, the design of an optimal fusion strategy remains an open research question. In this work, we present \method, an intuitive, simple, efficient and theoretically grounded framework that unifies graph and language representations. 
It is built around three theoretically-justified design components: (D1) \mechanismone, which enables robust LM training; (D2) \mechanismtwo, which encodes neighborhood information long exploited by classical link prediction heuristics, thereby enriching structural signals in a learnable way; and (D3) \mechanismthree, which reduces computational complexity. Together, these mechanisms establish an end-to-end cross-modal framework that scales to large graphs while capturing richer graph–text interactions. We conduct an extensive empirical evaluation, comparing against 40 existing models across 9 diverse datasets. Experiments show that \method consistently outperforms state-of-the-art approaches, delivering up to a 43.06\% improvement in Mean Reciprocal Rank (MRR) and setting a new benchmark for language–graph integration. 

\end{abstract}


\maketitle

\newcommand\blfootnote[1]{%
  \begingroup
  \renewcommand\thefootnote{}\footnote{#1}%
  \addtocounter{footnote}{-1}%
  \endgroup
}
\blfootnote{This paper was accepted for publication in ACM Transactions on Knowledge Discovery from Data (TKDD) on 2 September 2026. This is the authors' preprint version; the final version will appear in the ACM Digital Library.}

\section{Introduction}

Link Prediction (LP), which estimates the likelihood of an edge between two nodes in a graph, underpins applications such as protein-protein interaction prediction \citep{10.1093/nar/gky1131}, recommendation systems \citep{10.1145/1065385.1065415} and knowledge graph completion~\citep{10.1145/3366423.3380027}. Formally, LP is often cast as a ranking problem: given a candidate edge $e_{ij}$ between nodes $i$ and $j$ in a graph with adjacency matrix $\mathbf{A}$, the objective is to assign a score reflecting how likely it is that $i$ and $j$ are connected. The dominant paradigm follows an encoder-decoder design:  vector representations of nodes $i$ and $j$ are learned first and then a decoder (e.g., dot product) is used to score candidate links.   Early approaches for LP relied on heuristic measures such as common neighbors and shortest paths \citep{ADAMIC2003211,Zhou_2009}. Modern methods are dominated by GNNs, which aggregate node features across graph topology to produce embeddings in an end-to-end manner. A vanilla GNN approach to LP employs a GNN backbone for node embeddings, followed by a decoder applied to embedding pairs, such as dot product or hyperbolic distance \citep{liu_hyperbolic_nodate,Zhu2024OnTI}. To capture richer pairwise structural dependencies,  SEAL \citep{zhang_link_2018} and BUDDY  \citep{chamberlain_graph_2023} introduce techniques for aggregating local structural embeddings, including relative distance and Jaccard similarity.  

\begin{wrapfigure}{r}{0.45\textwidth}
    \centering
    \includegraphics[width=\linewidth]{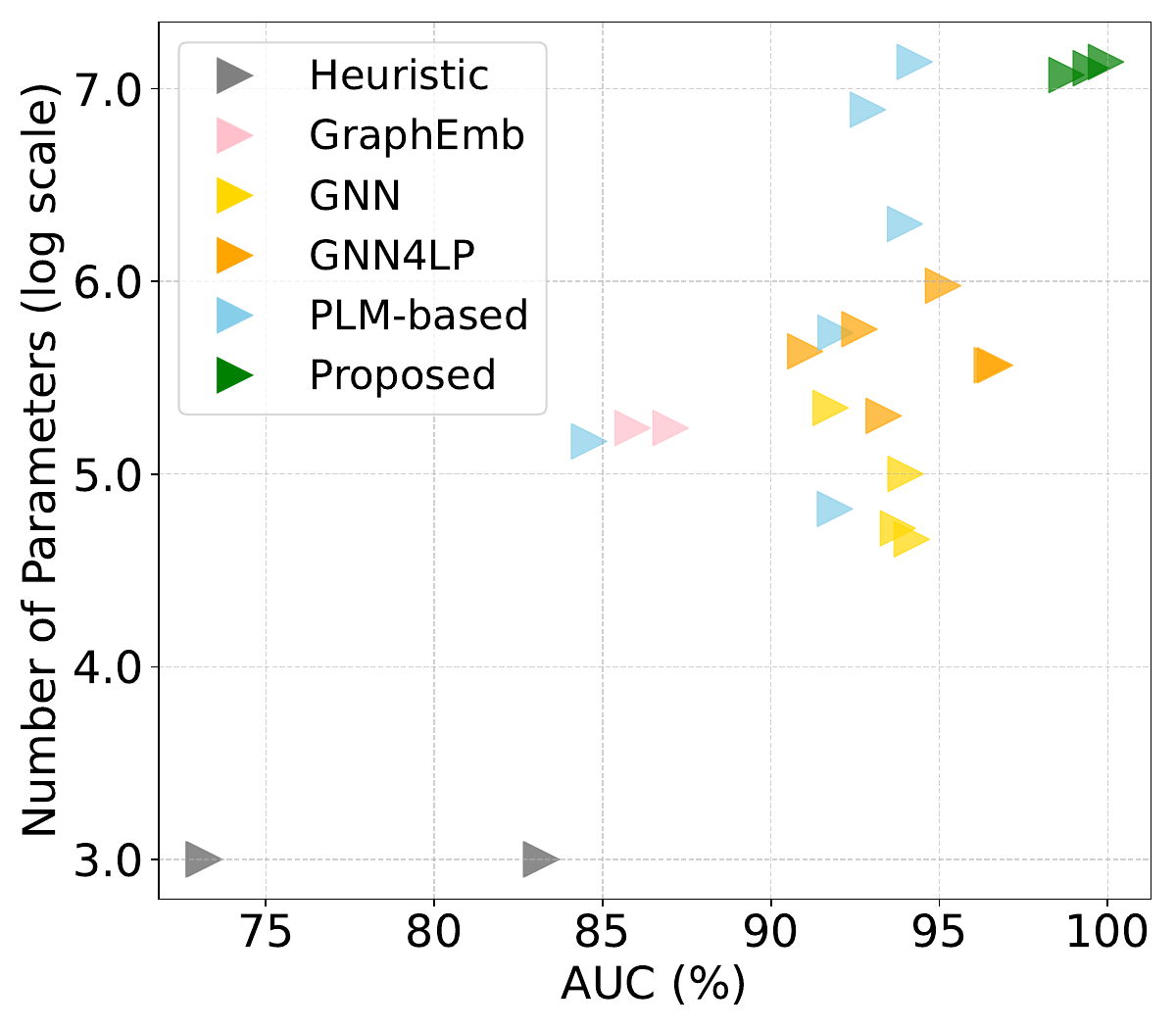}
    \caption{
    \method maintains accuracy (x-axis) while improving efficiency (y-axis)
    by 5\% over our most competitive baseline on GNN4LP and PLM-based baselines.
    \S \Cref{sec:benchmark}}
    \label{fig:benchmark_auc}
\end{wrapfigure}

In many real-world graphs, nodes are attributed with textual information, forming \textit{text-attributed graphs} (TAGs), which provide complementary semantics beyond structural properties. For example, in e-commerce graphs, user reviews or product descriptions serve as textual attributes, while edges may indicate co-purchase, co-viewed or other relations~\citep{yan2023a, nerrise2025grafs}. Yet standard GNN pipelines underutilize such textual information by simply representing it as a histogram of keywords or shallow semantic embeddings \citep{mikolov2013efficientestimationwordrepresentations}.

To better exploit the semantic potential of TAGs, recent studies have started to fuse GNNs with advanced language models (LMs) to integrate graph topology and semantics for link prediction. For instance, abstracts and titles in academic citation networks can be better embedded by sentence-embedding models such as MiniLM-L6-v2~\citep {Reimers2019SentenceBERTSE}, yielding more discriminative, robust and context-sensitive sentence representations than traditional static encoders \citep{mikolov2013efficientestimationwordrepresentations}. 

This research suggests the promise of a cross-modal framework, but the results are generally inconclusive due to two challenges. First, the performance gains rely on impractically expensive models that require jointly encoding pairs of texts with deep LMs and GCNs, greatly increasing the computational cost of inference. For instance, it could range from a few seconds to several hours and scaling from $10^3$ to $10^7$ parameters on the same dataset \citep{duan2024simteg, Chien2021NodeFE} (~\Cref{fig:benchmark_auc}). Second, recent studies consistently demonstrate that the cross-modal framework relying on a large LM does not necessarily improve performance over GNN alone; some preliminary studies also fail to outperform the current state-of-the-art technologies such as BUDDY and NCNC \citep{Chen2024TextspaceGF, Zhang2024TAGATG, yan2023a}.

In this work, we introduce \method, a novel cross-modal architecture for link prediction that simultaneously maximizes prediction accuracy (e.g., MRR, AUC, Hits@k) and maintains parameter efficiency and robustness (\Cref{fig:benchmark_auc,tab:compare-gcn,tab:compare-plm,tab:llm_result}). Our approach fuses three complementary graph signals through an MLP: 
(1) \textbf{\signalone}, capturing pairwise neighborhood connectivity; 
(2) \textbf{\signaltwo}, encoding edge-level textual or feature semantics; and 
(3) \textbf{\signalthree}, reflecting information aggregated from multi-hop neighbors through message passing. We fuse them through an MLP for link likelihood prediction and further propose three key designs that not only enhance cross-modal learning but also improve model expressiveness:
\textbf{(Design 1}) \mechanismOne \xspace for robust cross-modal training;
\textbf{(Design 2}) \mechanismTwo \xspace for compressing hidden embeddings; and
\textbf{(Design 3}) \mechanismThree \xspace for reducing training memory requirements. While variants of these components have been integrated in prior frameworks~\citep{Xu2018RepresentationLO, Wang2023NeuralCN, Chen2022CacheAugmentedII},
we are the first to theoretically justify their joint necessity within cross-modal iterative frameworks and  validate their effectiveness empirically across diverse datasets.

By leveraging these designs, \method is (i)~\textbf{effective},  seamlessly integrating text encoders and structural representations end-to-end; (ii)~\textbf{principled}, with theoretical guarantees supporting its design principles (D1–D3); and 
(iii)~\textbf{general}, functioning as a plug-and-play alternative to GCNs that integrates trivially with any PLM at no additional computational cost. Our contributions are as follows:\raggedbottom
\begin{enumerate}[left=0pt, labelsep=0.2em, itemsep=0.0em]

\item \textbf{\method: An End-to-End Cross-Modal Framework.} 
We propose \method, a theoretically motivated and computationally efficient framework that integrates graph and text modalities.  
Our method integrates three key design components, which together enable richer cross-modal representation learning than existing fusion strategies. 
\hfill (\S\Cref{sec:method}, ~\Cref{img:method}) 
\item \textbf{Theoretical Justification.}  
We provide theoretical justification for the effectiveness of our designs, which collectively enhance learning stability and efficiency, leading to more effective cross-modal representations. Furthermore, we show that they not only improve the performance of \method, but also serve as plug-in modules that consistently enhance existing GCN and PLM architectures. \hfill (\S ~\Cref{theorem:d1}, ~\Cref{prop:compact}, ~\Cref{thm:d3})
\item \textbf{Comprehensive Evaluation and Insights.}  We empirically evaluate \method against over 40 baselines on 9 TAG datasets under consistent experimental settings, establishing a standardized benchmark for link prediction. \method achieves state-of-the-art performance across all datasets, yielding up to 43.06\% improvement in MRR. Our results highlight the generalization limits of GCNs and demonstrate the benefits of incorporating PLMs. All code, data and benchmarks are publicly available at 
\href{https://github.com/ChenS676/TAG4LP}{[TAG4LP]}.  
\hfill (\S\Cref{sec:benchmark}, ~\Cref{tab:compare-gcn}, ~\Cref{tab:compare-plm})
\color{black}
\end{enumerate}

\begin{figure}
    \centering
    \includegraphics[width=\linewidth]{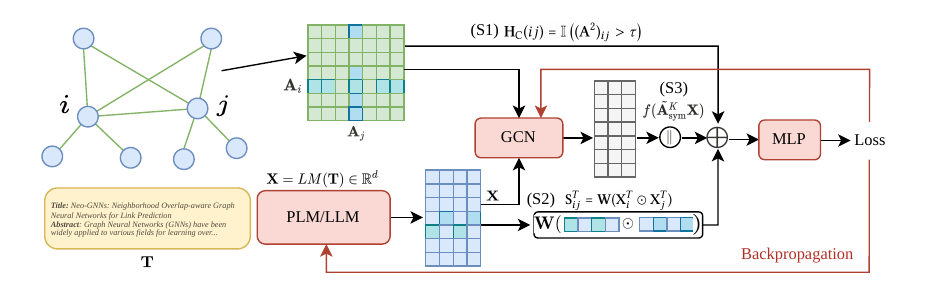}
    \caption{\small
     Overview of \method.
     The framework consists of three graph signals:
     \textbf{S1} \signalone \xspace \modify{from the graph structure (common neighbors, Eq.~\ref{equ:C2})};
     \textbf{S2} \signaltwo \xspace \modify{from PLM text embeddings};
     \textbf{S3} \signalthree from GCNs;
     and the final step concatenates \textbf{S1-3} through an MLP to perform link prediction.}
    \label{img:method}
\end{figure}


\section{Related Work}
\label{sec:related work}

Our work is related to approaches that fuse LLMs and GNNs for text-attributed graphs, as well as methods for link prediction. We will review these directions next.

\noindent \textbf{LLM-GCN Fusion Approaches for TAGs.} Recent research on TAGs has increasingly focused on bridging the gap between the semantic richness of LLMs and the structural modeling capabilities of GCNs. We categorize current approaches based on their text encoders, decoders (e.g., MLP, GCN) and optimization frameworks (e.g., Zero-Shot, Fine-Tune), as summarized in~\Cref{tab:comparison-related-work}.  

\textit{Text encoders} fall into two broad categories based on parameter scale: (1)~Shallow text embeddings: Early preprocessing methods transformed text attributes into statistical features and shallow representations~\citep{Harris1954DistributionalS, mikolov2013efficientestimationwordrepresentations}. (2)~PLM-based Encoders: These models capture higher-order semantics and generate sentence-level node embeddings. They are typically optimized using robust objectives such as contrastive fine-tuning with hard negatives and domain-adaptive pretraining, which allows for reliable transfer to specialized corpora with minimal supervision.  

PLM-based approaches can be grouped into three categories based on their \textit{optimization strategy}: (a)~Zero-Shot (ZS): The PLM is frozen and its pre-generated text embeddings are used as fixed inputs for graph aggregation (e.g., SimTeG~\citep{duan2024simteg}, GAINT~\citep{Chien2021NodeFE}, TAPE~\citep{he2023harnessing}). 
(b)~Fine-Tuned (FT): The PLM is integrated with the GCN and optimized end-to-end on a specific dataset, as seen in Graphformer~\citep{Yang2021GraphFormersGT} and Engine~\citep{Zhu2024EfficientTA}. 
(c)~Instruction-Tuned: Structural knowledge is incorporated into LLMs via prompt- or instruction-based tuning, as in GraphGPT~\citep{Tang2023GraphGPTGI} and LinkGPT~\citep{He2024LinkGPTTL}. \modify{A recent line of work also explores injecting graph structure into LLMs at the tokenizer level rather than through prompt/instruction tuning, e.g., the anchor-based node tokenizer of \citet{AnchorsToAnswers2025}.}
Additionally, \citet{Chen2024TextspaceGF} observed that naively combining LLMs and GCNs in an iterative architecture without task-specific design fails to yield meaningful performance gains.
\textit{We address this gap with \method, the first framework to theoretically unify PLM-based semantics with multi-level structural signals, providing a principled and computationally efficient solution. }
\modify{A related but distinct line of work targets \emph{dynamic} text-attributed graphs, where node/edge structure and text evolve over time: \citet{Yu2023DyGLib} propose a unified architecture and library for dynamic graph learning, \citet{Zhang2024DTGB} introduce DTGB, a benchmark specifically for dynamic TAGs, and \citet{Xu2025DynamicTAGMultiModal} study multi-modal representation learning on dynamic TAGs. \method{} targets the static setting; we discuss how our three signals could be extended to this dynamic setting in \S\ref{sec: conclusion}.}

\begin{wraptable}{r}{0.4\textwidth}
\footnotesize
\caption{Comparison with existing methods. \bymark: public benchmark, \cmark: benchmarked model, Zero-Shot: ZS, Fine-Tune: FT}
\label{tab:comparison-related-work}
\begin{tabular}{lcccc}
    \toprule 
    \multirow{2}{*}{\textbf{Works}} & \multicolumn{2}{c}{\textbf{ZS}} & \multicolumn{2}{c}{\textbf{FT}} \\ 
    & \textbf{MLP} & \textbf{GCN} & \textbf{MLP} & \textbf{GCN} \\ 
    \midrule
    Graphformer \citep{Yang2021GraphFormersGT} & & & & \bymark \\
    GAINT \citep{Chien2021NodeFE} & \bymark & \bymark & & \\
    SimTeG \citep{duan2024simteg} & \bymark & \bymark & & \\
    TAPE \citep{he2023harnessing} & & \bymark & & \\
    LEADING \citep{Xue2023EfficientEL} & & \bymark & & \\
    ENGINE \citep{Zhu2024EfficientTA} & & & & \bymark \\
    TEG-DB \citep{li2024tegdb} & \bymark & \bymark & & \\
    \midrule
    Ours & \cmark & \cmark & \cmark & \cmark \\
    \bottomrule
\end{tabular}
\end{wraptable}

\noindent \textbf{Message Passing in Link Prediction.} GCNs are traditionally used to generate initial node embeddings for link prediction, followed by a dot-product decoder \citep{Wang2023NeuralCN, 10.1145/3637528.3672025, li_evaluating_nodate-1}. They operate via a message-passing mechanism, in which node representations are iteratively refined by aggregating features from local neighborhoods. However, GCNs have inherent theoretical limitations in distinguishing structurally identical nodes and exhibit limited expressivity in capturing critical pairwise relationships \citep{srinivasan_equivalence_2020, Zhang2020LabelingTA}. Methods such as SEAL and NBFNet enable message passing to learn pairwise information specific to each target link \citep{Zhang2020LabelingTA, chamberlain_graph_2023}. Nevertheless, these approaches incur high training costs, as they require constructing a separate local subgraph and extracting features for every target link. To address this issue, recent methods decouple message passing from pairwise information modeling. Representative examples include NCNC \citep{Wang2023NeuralCN}, which exploits common-neighbor information, as well as BUDDY \citep{chamberlain_graph_2023}, HL-GNN \citep{zhang2024heuristiclearninggraphneural} and Neo-GNN \citep{yun2021neognns}, which incorporate heuristic global structural information. \modify{Beyond text-attributed benchmarks, link prediction has also been extended to modality-rich recommendation settings, e.g., \citet{FROG2025} recently proposed a modality-aware friend-recommendation framework that couples user preferences with heterogeneous signals.} \textit{Our work aims to enhance the efficiency of feature learning between textual and structural information, without increasing the overall expressivity of existing link prediction frameworks.}
\color{black}

\section{\method: A New Efficient, Simple and Theoretically Justified Model}
\label{sec:method}
In this section, we present our framework, \method. We begin by introducing the three core data signals that underpin \method: \signalone, \signaltwo, and \signalthree (\Cref{subsec:data-signals}). Next, in \Cref{subsec:Vallina Framework}, we describe our efficient iterative architecture and its three complementary design mechanisms (\Cref{img:method}), which we justify both theoretically and empirically. Finally, in \Cref{sec:complexity}, we analyze the computational complexity of \method and highlight its efficiency advantages over existing approaches.


\subsection{Data Signals in \method}
\label{subsec:data-signals}
The data signals in our framework are motivated by the \emph{latent space model}, a foundational paradigm that has been widely adopted across disciplines such as sociology and statistics \citep{Hoff01122002}.
The latent space model provides a principled interpretation of graph formation, where edge probabilities are governed by distances between nodes embedded in an unobserved metric space. Within this framework, prior work establishes quantitative links between observable graph signals and latent distances. In particular, \citet{Hoff01122002} shows that the number of common neighbors between two nodes tightly bounds their latent distance. More recently, \citet{mao2023revisiting} demonstrates that when higher- and lower-order structural signals are absent, node feature similarity, quantified by feature dot products, dominates latent distance estimation. In this regime, they further prove that the common-neighbor count becomes, with high probability, negatively correlated with feature proximity, revealing an intrinsic tension between structural and feature-based signals in latent graph models.

Guided by these theoretical connections, we define three complementary data signals that serve as proxies for latent relationships: Our approach fuses three complementary graph signals through an MLP: (S1) \signalone, capturing local graph structure;  (S2) \signaltwo, encoding node-level textual semantics; and (S3) \signalthree, which goes beyond 1-hop neighborhoods and measures how similarity diffuses through multi-hop connections. We present each one in more detail below.

\noindent  {\textbf{(S1) \signalone }} \color{black} We utilize \modify{a thresholded indicator of} Common Neighbors, a first-order structural feature that solely utilizes graph topology:
\begin{equation}
\label{equ:C2}
\mathbf{H}_{\text{C}}(ij)=  \mathbb{I}\left((\mathbf{A}^2)_{ij} > {\color{blue}\tau} \right),
\end{equation}
where \( \mathbf{A} \in \{0, 1\}^{n \times n} \) is the binary adjacency matrix, $\mathbb{I}$ is the indication 
function, and \modify{\( \tau \)} is a hard threshold to retain stronger connections \citep{Newman2001ClusteringAP, Wang2023NeuralCN}. 
\modify{$\tau$ denotes the CN threshold introduced.}
\modify{We follow \citet{Wang2023NeuralCN} in adopting the thresholded indicator form. Thresholding to a binary indicator (i) helps remove noisy
edges introduced by spurious connections, and (ii) removes the hub bias that arises when two nodes are connected only through a shared global hub, which would otherwise mislead the prediction. We report its sensitivity in \Cref{sec:ablation study}.}

\noindent {\textbf{(S2) \signaltwo}.} 
The second signal captures the {semantic proximity} between nodes based on their textual content alone. The intuition is that similar nodes exhibit high similarity in their language representations, even without explicitly considering the graph structure. Let $\mathbf{T}$, $\mathbf{X}^T = \text{PLM}(\mathbf{T})$ represent the raw text and embedded text features from a pretrained language model. PLM utilizes Transformer backbone to extract deep linguistic features by modeling intra-node dependencies. The intuition is that similar nodes exhibit high similarity in their language representations, even without explicitly considering the graph structure. The semantic proximity between node pair $(i, j)$ is defined as:  
\begin{equation}
    \mathbf{S}_{ij}^{T} = \mathbf{W}(\mathbf{X}^T_i \odot \mathbf{X}^T_j),
\end{equation}
where \( \mathbf{W} \) is a learned weight matrix and the operator $\odot$ denotes the Hadamard product. Compared to fixed similarity measures such as cosine similarity \citep{yang2016revisitingsemisupervisedlearninggraph}, this learnable formulation allows the model to adaptively weight semantic dimensions that are most predictive of latent relationships\citep{wang2021heterogeneousgraphattentionnetwork}. 

\noindent {\textbf{(S3) \signalthree}.}
While structural and semantic (feature) proximity primarily capture local or pairwise relationships, they may overlook higher-order dependencies that emerge through multi-hop interactions. To address this, we introduce the \emph{propagated proximity} signal, which captures how node similarity evolves as information diffuses through the graph. This signal reflects the intuition that even if two nodes are not directly connected, their embeddings may become more aligned after several rounds of message passing, analogous to reduced latent distances through indirect ties. It is especially useful in capturing the information of similar neighbors when the structure exhibits homophily \citep{lee2024netinfof, Kipf2016SemiSupervisedCW}.

\modify{Empirically, S3 remains consistently beneficial across all our TAG benchmarks --- every Hits@50 / MRR / AUC cell in \Cref{tab:ablation_study} improves when S3 is included; see \S\ref{sec:ablation study} for the full comparison of \method with and without S3 across all nine metrics.}
\color{black}
Formally, we propagate the input features over $K$ steps using a symmetrically normalized adjacency matrix with self-loops:
\begin{equation}
\label{equ: aggregated feat}
\mathbf{H}^{K} = f(\tilde{\mathbf{A}}_{\text{sym}}^K\mathbf{X}),
\end{equation}
where $\mathbf{H}^{K}$ represents the final embedding from GCN in iteration $K$,   $\tilde{\mathbf{A}}_{\text{sym}}=(\mathbf{D+I})^{-\frac{1}{2}}(\mathbf{A} + \mathbf{I}) (\mathbf{D+I})^{-\frac{1}{2}}$ is the symmetrically normalized adjacency matrix with self-loops, $\mathbf{I}$ is the identity matrix and $\mathbf{D}$ is diagonal degree matrix.

\subsection{\method: Efficient Iterative Architecture}
\label{subsec:Vallina Framework}
\color{black}
\method{} is a lightweight iterative architecture that integrates the three complementary signals introduced in the previous section through a unified cross-modal framework. To effectively exploit these signals while maintaining efficiency, \method{} incorporates three design mechanisms (D1-D3) that enhance the base architecture. We first present the vanilla framework that forms the backbone of the model, and subsequently describe how these designs further improve its effectiveness and efficiency.

The vanilla \method{} framework combines the three data signals using simple linear transformations followed by component-wise nonlinearities, avoiding heavy fusion mechanisms. Concretely, the final prediction is computed as
\begin{equation}
\label{equ:final embed}
\mathbf{Y} = \mathrm{MLP}\left( [ \mathbf{H}^K \,;\, \beta \mathbf{H}_C \,;\, \mathbf{S}_T ] \right),
\end{equation}
where $[\,\cdot\,;\,\cdot\,]$ denotes concatenation and $\mathbf{Y}$ represents the model output, $\beta$ is the weight for threshold common neighbor. Figure~\ref{img:method} provides a high-level illustration of the overall architecture, while~\Cref{algo:method} summarizes the corresponding inference procedure. Notably, the architecture is intentionally simple: each signal is computed independently and fused only at the final stage, which avoids expensive intermediate interactions and enables efficient deployment.
\color{black}
Building upon this backbone, we introduce three complementary designs (D1-D3) that further improve robustness across diverse settings while substantially reducing memory overhead.

\subsubsection{\bf D1: \mechanismOne.}  
\label{subsec: efficient design}
We propose design D1, a bypass connection that routes textual information directly to the final embeddings, as illustrated in \Cref{img:method}. \color{black} By acting as a semantic highway, D1 preserves the raw nuances of the PLM and ensures that the supervision signal from the loss function is effectively back-propagated to the language model. {This design enables the PLM to be optimized more robustly and efficiently within the cross-modal framework by maintaining a direct gradient flow from the loss objective to the textual representations.} 

To justify this design theoretically, we utilize a quantitative measure, the \emph{influence score} (\Cref{def:influence_score}), which characterizes how the PLM parameters affect the final likelihood objective. This notion is inspired by influence functions in statistics \citep{koh2020understandingblackboxpredictionsinfluence}. Using this metric, we analyze the sensitivity of the link likelihood objective with respect to the PLM parameters, and show that D1 systematically strengthens this sensitivity.

\noindent
\begin{definition}[Influence Score]
\label{def:influence_score}
Let $\mathbf{x}$ denote the input text feature, $\mathbf{z}$ the hidden representation at an intermediate layer and 
$\mathbf{y} \in \mathbb{R}^{|\mathcal{V}| \times d}$ the final node embeddings before the decoder. Let $\mathbf{h}^{(0)}$ and $\mathbf{h}^{(k)}$ denote the node representations at the input and the $k$-th layer, respectively.  We define the \emph{influence score} of the initial text embedding of node $x$ on the final embedding of node $y$ as
\[ I_x(y) = \frac{\mathbf{e}^\top \left| \frac{\partial \mathbf{h}^{(k)}_y}{\partial \mathbf{h}^{(0)}_x} \right| \mathbf{e} }{ \sum_{z \in \mathcal{V}} \mathbf{e}^\top \left| \frac{\partial \mathbf{h}^{(k)}_z}{\partial \mathbf{h}^{(0)}_x} \right| \mathbf{e} }, \]
where $|\cdot|$ denotes element-wise absolute value and $\mathbf{e}$ is the all-ones vector of compatible dimension.
\end{definition}

\noindent \textbf{Theoretical Justification.}  
We show that integrating D1: \mechanismone \xspace in the cross-modal framework, which directly links PLM text embeddings to the final loss, leads to a higher influence score compared to a configuration without it. We analytically compare the influence scores of the PLM parameters with and without D1 in \Cref{theorem:d1}. The complete proof is provided in \Cref{app:d1}. 
\begin{theorem}[Influence Enhancement]
\label{theorem:d1}
Let $\mathbf{h}^{(0)}_v \in \mathbb{R}^d$ denote the initial embedding of node $v$ produced by the PLM, and $\mathbf{h}^{(k)} = f(\mathbf{h}^{(0)}, \mathcal{G})$ denote the output of a $k$-layer GCN, where $f$ is continuously differentiable. Then, for any nodes $x, y \in \mathcal{V}$, the influence score defined in \Cref{def:influence_score} satisfies $I_{\mathrm{with}\;D_1}(x,y) \;=\; I_{\mathrm{without}\;D_1}(x,y)\quad\text{for } y\neq x,$, while on the PLM-direct (self) path
$$I_{\mathrm{with}\;D_1}(x,x)\ge I_{\mathrm{without}\;D_1}(x,x)$$
 , with a strict increase of $\|B\|_1>0$ whenever $B\neq 0$ and the bypass does not cancel the existing self-path (entrywise sign-compatibility $J_iB_i\ge0$, where $J=\partial\mathbf{h}^{(k)}_x/\partial\mathbf{h}^{(0)}_x$ and $B$ is the trainable bypass matrix of design~D1). See \Cref{prop:d1-influence} in Appendix~\ref{app:d1}. 
\end{theorem}
\begin{figure}[htbp]
    \centering
    \includegraphics[width=0.6\linewidth]{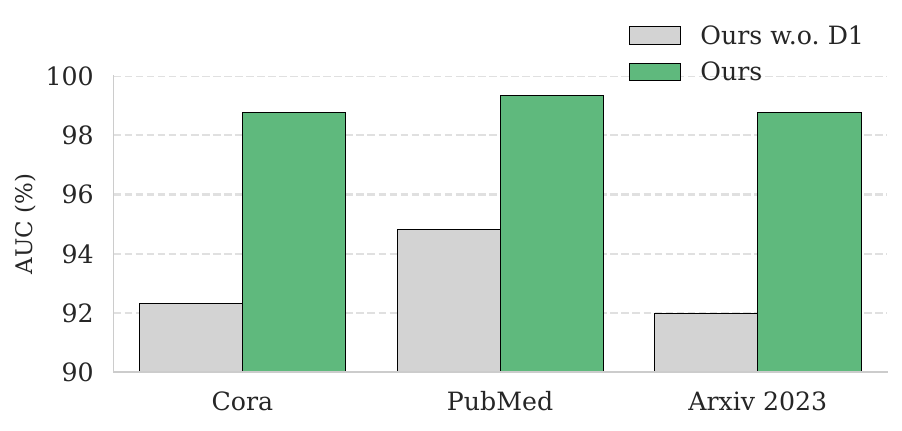}
    \caption{\textbf{Ablation study of D1.} Comparison of link prediction performance (AUC) with and without the proposed bypass connection. Across three benchmark datasets, the absence of D1 results in a performance degradation of 4\%–6.8\%, validating its role in facilitating effective gradient flow and preserving PLM semantic nuances.}
    \label{fig:d1_ablationstudy}
\end{figure}
We show that incorporating D1, bypass connection strictly increases the influence term $I$ of the final link prediction loss on the PLM embeddings.  
The increased influence score ensures that the PLM receives sufficiently strong gradients to adapt its semantic knowledge specifically for the link prediction task.

\noindent \textbf{Observations.}
To empirically validate the effectiveness of design D1, we conduct an ablation study of \method{} with and without D1, as illustrated in \Cref{fig:d1_ablationstudy}.  We observe that \method{} (i.e., models equipped with design D1) consistently achieves higher link prediction performance than those without it. It increases the AUC by up to 6\% on Cora and by 4\% and 6.8\% on Pubmed and Arxiv\_2023 respectively. Although \Cref{theorem:d1} is derived under specific assumptions, our empirical results demonstrate that its conclusions hold broadly in practice.  

\subsubsection{\textbf{D2: \mechanismTwo.}} 
The core objective of design D2 is to mitigate the high-dimensionality requirement inherent in the vanilla framework (\Cref{subsec:Vallina Framework}), where the hidden dimension $d$ scales linearly with the number of vertices ($|\mathcal{V}|$). The underlying principle of D2 lies in the near-orthogonality between the structural proximity signal \textbf{S1} ($\mathbf{H}_C$) and the semantic proximity signal \textbf{S2} ($\mathbf{S}_T$), where their non-zero elements exhibit minimal overlap \citep{mao2023revisiting}. Given that the non-zero entries of $\mathbf{H}_C$ and $\mathbf{S}_T$ are largely disjoint, transitioning from the high-dimensional concatenation used in the vanilla framework to an additive form is effectively lossless. D2 reduces the hidden dimension from $d = |\mathbf{H}^K| + |\mathcal{T}| + |\mathcal{V}|$ (as introduced in vanilla framework \Cref{sec:complexity}) to $d = |\mathbf{H}^K| + 1$. Summing these signals is approximately equivalent to blending information from separate channels into a single dense tensor without significant signal interference.

\noindent \textbf{Theoretical Justification.}
\modify{We provide a theoretical explanation for why the structural signal $\mathbf{H}_C$ and the 
semantic signal $\mathbf{S}_T$ can be compressed into a substantially lower-dimensional representation while preserving the 
information relevant for link prediction. 
Our analysis is conducted under a standard \emph{latent space model} \citep{Sarkar2011theoretical}, where each of the $N$ nodes is associated with an $r$-dimensional latent position. We assume that link formation is governed by two complementary mechanisms. If two nodes are sufficiently close in the latent structural space, their connection probability is primarily determined by structural proximity, resulting in similar common-neighbor patterns. Otherwise, when structural proximity is weak, semantic proximity becomes the dominant factor, allowing structurally distant but semantically similar nodes to form links. This assumption is consistent with empirical observations on widely used benchmark graphs \citep{mao2023revisiting}. Under this setting, we show that the joint structural-semantic signal admits a compact representation with negligible information loss.}

\begin{proposition}[Lossless Additive Compression of $\mathbf{H}_C$ and $\mathbf{S}_T$]
\label{prop:compact}
For each node pair \(a\), let \(\mathbf h^a=[\mathbf H_C]_a\) and \(\mathbf s^a=[\mathbf S_T]_a\), with \(z^a_{\mathrm{cat}}=(\mathbf h^a,\mathbf s^a)\) and \(z^a_{\mathrm{sum}}=\mathbf h^a+\mathbf s^a\). For any node pairs \(a,b\),
\[
\big|\langle z^a_{\mathrm{sum}}, z^b_{\mathrm{sum}}\rangle - \langle z^a_{\mathrm{cat}}, z^b_{\mathrm{cat}}\rangle\big| \;\le\; \big|\langle \mathbf{h}^a, \mathbf{s}^b\rangle\big| + \big|\langle \mathbf{s}^a, \mathbf{h}^b\rangle\big|.
\]
Thus, the Gram matrices induced by summation and concatenation coincide when the structural and semantic signals have disjoint supports, making the compression lossless for inner-product-based predictors. More generally, the distortion is controlled by their cross-signal overlap. Since structural and semantic proximity rarely co-occur strongly \citep{Sarkar2011theoretical, mao2023revisiting}, this overlap is typically small in practice (\Cref{fig:d2_clarification}), motivating additive compression with negligible information loss.
\end{proposition}

\begin{figure}[h]
    \centering
    \includegraphics[width=0.9\linewidth]{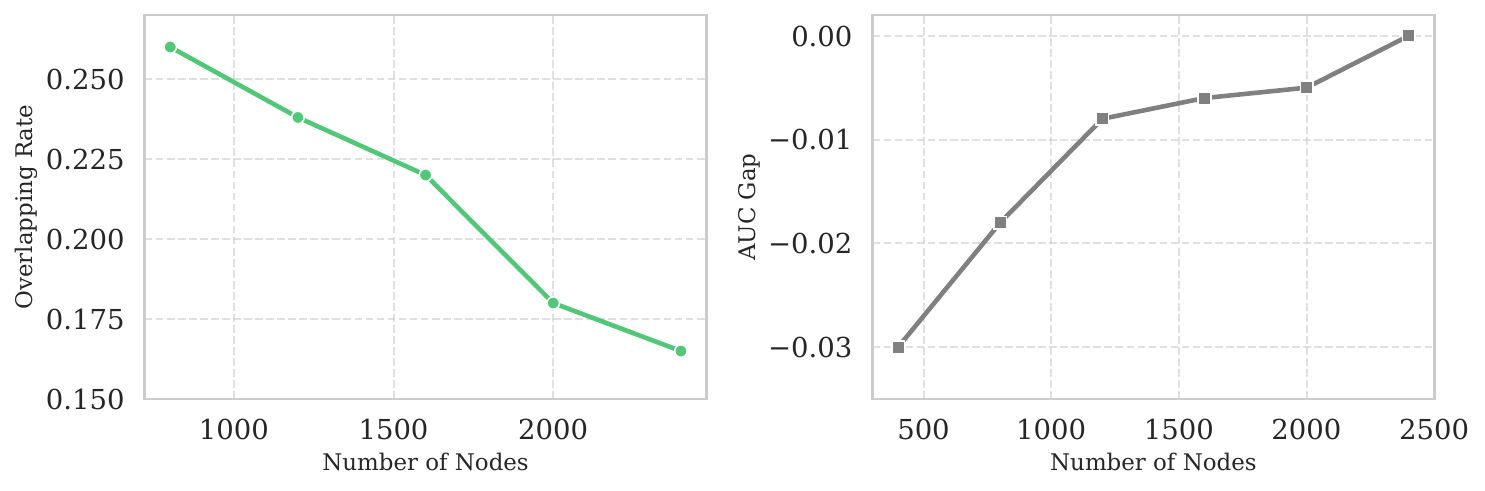}  
    \caption{\textbf{Left:} Overlapping rate between $\mathbf{H}_C$ and $\mathbf{S}_T$; \textbf{Right:} AUC gap between models with and without compression.
    The results show that $\mathbf{H}_C$ and $\mathbf{S}_T$ are approximately compressive and that the framework equipped with D2 converges to a similar performance level as the non-compressed variant.}
    \label{fig:d2_clarification}
\end{figure}

\noindent \textbf{Observations.} {In our empirical analysis, we define the \textit{overlapping rate} to quantify the overlap of non-zero elements. As illustrated in \Cref{fig:d2_clarification} (left), overlapping rate decays approximately linearly as a function of the number of sampled vertices $|\mathcal{V}_s|$. Furthermore, we evaluate the performance gain of the D2 component by defining the \textit{AUC gap}: $\Delta_{\text{AUC}} = \text{AUC}(\mathcal{M}_{\text{with D2}}) - \text{AUC}(\mathcal{M}_{\text{without D2}})$. As shown in \Cref{fig:d2_clarification} (right), $\Delta_{\text{AUC}}$ converges toward zero as $|\mathcal{V}_s|$ increases. This indicates that the proposed summation achieves equivalent predictive performance as the vanilla framework, but with substantially reduced computational complexity.}

\subsubsection{\bf D3: \mechanismThree.} 
The third design, D3, addresses the significant computational and memory overhead inherent in generating initial node features for GCN-based architectures. In a joint training setting, these features act as a critical representational bridge between the PLM and the GCN. As established in our complexity analysis, encoding these features requires $O(|\mathcal{V}| \times |\mathcal{T}| \times d)$ per training step. Unlike traditional approaches that keep the PLM frozen to bypass this cost, our method maintains the trainability of both components, necessitating a more efficient integration strategy. {
In D3, we adopt a dynamic cache-embedding strategy. For each minibatch of edges, we utilize the PLM to encode only the vertices involved in those specific edges, while replacing the features of all other nodes with cached embeddings. Mechanically, we first initialize a global feature matrix $\mathbf{X} \in \mathbb{R}^{|\mathcal{V}| \times d}$ by pre-computing all node embeddings via the PLM and storing them in a persistent cache. During each training iteration, we perform a live re-encoding \emph{only} for the textual descriptions of the active nodes $(i, j)$ in the current mini-batch. The input feature matrix for the GCN is then constructed by substituting the stale cached vectors with these fresh, gradient-carrying embeddings:}
\begin{equation}
    \mathbf{X}_{in} = \text{Concat}(\mathbf{X}_{cache \setminus \{i,j\}}, \mathbf{h}_i, \mathbf{h}_j)
\end{equation}
where $\mathbf{h}_i$ and $\mathbf{h}_j$ represent the real-time PLM outputs. This strategy reduces the per-batch encoding complexity from $O(|\mathcal{V}||\mathcal{T}|d)$ to $O(b|\mathcal{T}|d)$ (where $b$ is the batch size), effectively decoupling the PLM's computational load from the total graph size. While cached embeddings have been explored in general GNN contexts \citep{Zhu2020GeneralizingGN, Wang2023NeuralCN}, our implementation specifically optimizes this mechanism within the D3 architecture to facilitate scalable, end-to-end PLM-GCN integration.

\noindent \textbf{Theoretical Justification.} The validity of D3 rests on the fact that while cached embeddings are technically "stale" (i.e., computed using PLM parameters from previous iterations), they remain a high-fidelity approximation of the current PLM state. We show that D3 provides a provably controlled approximation to full-parameter PLM updates by bounding the \emph{staleness error}. Specifically, for any node pair $(i,j)$, the discrepancy between the true gradient (obtained via full re-encoding) and the D3 gradient depends on how recently the nodes in $\mathcal{V} \setminus \{i,j\}$ were updated in the cache. \modify{The resulting cache error is governed by how far the PLM parameters have drifted since each node was last refreshed; it stays small when late-training updates are small or the cache is periodically refreshed, so the cached feature matrix $\mathbf{X}$ closely tracks the representation that the current PLM parameters would produce.} We theoretically justify this in \Cref{thm:d3}. 

\begin{theorem}[Gradient Staleness Bound for D3]
\label{thm:d3}
\modify{Let $g^{\mathrm{full}}_t$ and $g^{\mathrm{cache}}_t$ denote the gradients with respect to the GCN/predictor parameters $\phi$ under full PLM re-encoding and the D3 cache-embedding scheme at iteration $t$. Under a Lipschitz PLM encoder and a standard smoothness assumption on the feature-to-gradient map (assumptions (A1)--(A2) in Appendix~\ref{app:d3}),
\[
\|g_t^{\mathrm{full}}-g_t^{\mathrm{cache}}\|
\le C\sum_{v\in\mathcal V}
\|\theta_t-\theta_{r_v(t)}\|,
\]
where $r_v(t)$ is the last iteration at which node $v$ was refreshed and $C$ is a constant. Hence the error is uniformly bounded under a bounded step size with periodic refresh, and vanishes as training converges (summable updates) provided every node keeps being refreshed. See \Cref{thm:d3-staleness} in Appendix~\ref{app:d3}}
\end{theorem}
\noindent \textbf{Remark.} \modify{The discrepancy is governed by cache staleness, the parameter drift since each node's last refres, rather than by node coverage alone: it stays controlled when late-training updates shrink or the cache is periodically refreshed, so after sufficient training the D3 gradient is practically indistinguishable from the full re-encoding gradient and the optimization trajectory remains stable despite stale embeddings.}

\color{black}
\subsection{Time Complexity of \method} 
\label{sec:complexity}
\raggedbottom  
The overall computational complexity of \method{} is determined by two components: the {Transformer backbone (detailed in S2)} and the graph message-passing module.
For the Transformer, each of the $l$ layers consists of a feed-forward sub-layer with cost $O(|\mathcal{T}| d^2)$ and a self-attention sub-layer with cost $O(|\mathcal{T}|^2 d)$, where $|\mathcal{T}|$ is the token length per node and $d$ is the hidden dimension, giving a total of $O\!\left(l\left(|\mathcal{T}| d^2 + |\mathcal{T}|^2 d\right)\right)$. The graph message passing adds $O(d|\mathcal{E}|)$, where $|\mathcal{E}|$ is the number of edges. Thus, the overall runtime complexity is 
\[
O(d^2 |\mathcal{T}| + d|\mathcal{T}|^2 + d|\mathcal{E}|).
\]
The memory footprint follows a similar decomposition: storing feed-forward activations requires $O(l|\mathcal{T}| d)$, attention matrices require $O(l(|\mathcal{T}|^2 + |\mathcal{T}| d))$ and message passing requires $O(d|\mathcal{E}|)$, leading to a total memory complexity of 
\[
O\!\left(l\left(|\mathcal{T}| d + |\mathcal{T}|^2\right) + d|\mathcal{E}|\right).
\]
In both cases, the dominant terms depend on the trade-off between token length and graph sparsity: the quadratic attention term dominates for long token sequences, whereas the message-passing term dominates for dense graphs.

\color{black}

\section{Empirical Evaluation}
\label{sec:benchmark}
In our empirical analysis, we evaluate the performance of \method{} on nine text-attributed graphs to address the following research questions:
\begin{enumerate}[topsep=0pt, itemsep=0pt, parsep=0pt, partopsep=0pt]
    \item[\textbf{RQ1:}] On widely-used TAG benchmarks, where all existing GCN-based approaches are fairly compared, does \method{} with integrated text modality improve link prediction performance?
    \item[\textbf{RQ2:}] Compared with existing PLM-based approaches, does \method{} equipped with D1--D3 enhance graph structure learning?
    \item[\textbf{RQ3:}] 
    Does \method{} scale efficiently while consistently improving performance?
    \item[\textbf{RQ4:}] What is the impact of three data signals, LLM choices and LLM scales?
\end{enumerate}

\begin{table}[ht]
    \centering
    \color{blue}
    \caption{Graph statistics of the benchmark datasets across
    different scales. We use commonly used datasets Cora, Pubmed and Arxiv\_2023 for RQ1-3, and larger datasets for RQ4. The last two rows show the minimum and maximum text lengths (in tokens) and the total number of data points (in millions) for each dataset.
}
    \begin{adjustbox}{width=\textwidth}
    \begin{tabular}{lccccccccc}
        \toprule
        & \textbf{Pwc\_small} & \textbf{Cora} & \textbf{PubMed} & \textbf{Arxiv\_2023} & \textbf{History} & \textbf{Photo} & \textbf{Pwc\_medium} & \textbf{Ogbn\_arxiv} & \textbf{Citationv8}\\
        \midrule
        \textbf{\#Nodes $|\mathcal{V}|$} & 140 & 2708 & 19717 & 46198 & 36655 & 47420 & 86795 & 169343 & 1106759\\
        \textbf{\#Edges $|\mathcal{E}|$} & 798 & 10858 & 88648 & 77726 & 496248 & 872602 & 933411 & 2315598 & 12227452 \\
        \textbf{File Size (Mb)} & 0.139 & 2.57 & 31.01 & 57.26 & 57.12 & 37.44 & 99.58 & 190.27 & 1034.18 \\
        \textbf{\#Split} & R & R & R & R & R & R & R & R & R\\
        \midrule
        \textbf{Min-Max Length} & 79--293 & 8--929 & 4--994 & 7--1617 & 42--329 & 4--7281 & 16--610 & 20--2214 & 4--774\\
        \textbf{Data Points (M)} & 0.139 & 2.57 & 31.01 & 57.26 & 42.16 & 37.44 & 99.58 & 190.27 & 1034.18\\
        \bottomrule
    \end{tabular}
    \end{adjustbox}
    \label{tab:graph-simple-stats}
    \color{black}
\end{table}

\subsection{RQ1: Comparison with GCN-based Approaches}
\label{subsec: RQ1}
First, we assess the performance of \method and compare it to strong baselines of structure-agnostic approaches on three widely-used text-attributed graphs. 

\subsubsection{Empirical Setup for RQ1}
\noindent \textbf{Datasets.} We begin with a thorough benchmark on Cora, Pubmed and Arxiv\_2023 \citep{Hu2021OGBLSCAL} (see ~\Cref{tab:compare-gcn}) to provide a fair comparison with existing GCN-based approaches.

\noindent \textbf{Benchmarked Approaches.} 
We consider all methods listed in \Cref{tab:compare-gcn}, which encompass five major categories: \uline{1) Graph Heuristics:} Local heuristics leverage shared neighborhood information, including Common Neighbor (CN), Adamic-Adar (AA) and Resource Allocation (RA), while global heuristics consider all paths between connected nodes, such as Katz, Shortest Path and symmetric PageRank \citep{ADAMIC2003211, ilprints422}. \uline{2) Knowledge Graph Embeddings:} Approaches such as RotatE \citep{sun2019rotateknowledgegraphembedding}, TransE \citep{Zhang2022TranSTK}, ComplEx \citep{trouillon2016complexembeddingssimplelink} and DistMult \citep{yang2015embeddingentitiesrelationslearning} represent entities and relations as embeddings in a continuous vector space, where relational reasoning is performed through translation or bilinear transformations. \uline{3) Structure-based Methods:} These focus on proximity-preserving embeddings that model neighborhood contexts via random walks, including DeepWalk, Node2Vec and LINE \citep{perozzi_deepwalk_2014, tang_line_2015, grover_node2vec_2016}. \uline{4) Aggregation-based Methods:} Under the GCN paradigm, we consider GCN \citep{Kipf2016SemiSupervisedCW}, GAT \citep{Velickovic2017GraphAN}, SAGE \citep{Hamilton2017InductiveRL} and GIN \citep{Xu2018HowPA}. \uline{5) GCN4LP:} SEAL, BUDDY and ELPH are specifically designed for link prediction, enhancing GCNs’ ability to capture pairwise local structures through various labeling techniques and structural embeddings as auxiliary node features \citep{Zhang2020LabelingTA, zhang_link_2018, chamberlain_graph_2023}. 
NCNC further leverages triadic closure to mitigate distribution shifts introduced by feature-specific data splitting methods \citep{Wang2023NeuralCN}.

\noindent \textbf{Hyperparameter Settings.} We conduct an extensive hyperparameter tuning based on hierarchical grid search across \textit{all models} listed in  \Cref{tab:compare-gcn}. Tuning parameters include all existing parameters in the optimizer and GCN, namely learning rate, weight decay, the number of convolution layers, MLP layers, the number of heads in GAT, hidden neurons and beyond. We provide more details about hyperparameter tuning in the \Cref{app:hyper_tune}.  

\setlength{\tabcolsep}{3pt}
\begin{table*}[b]
  \centering
  \tiny
  \caption{\textbf{Comparison with GCN-agnostic approaches}. Mean $\pm$ Standard deviation of \method for Hits@50, Hits@100, and MRR metrics across 5 random splits on three commonly used datasets. The top 1-2 ranked models are highlighted in emerald. For more results, please refer to the \Cref{sec:app-full-result}.}
  \label{tab:compare-gcn}
  \footnotesize
  \arrayrulecolor{black}
  \begin{adjustbox}{width=\textwidth}
  \begin{tabular}{llccccccccc c}
    \toprule
     \multirow{2}{*}{Category} &\multirow{2}{*}{Models} & \multicolumn{3}{c}{Cora} & \multicolumn{3}{c}{PubMed} & \multicolumn{3}{c}{Arxiv 2023} & \\
     & & Hits@50 & Hits@100 & MRR & Hits@50 & Hits@100 & MRR & Hits@50 & Hits@100 & MRR & Rank \\
    \midrule
    \multirow{4}{*}{Embedding}  
    & RotatE   & 27.51{\scriptsize$\pm$6.39} & 56.60{\scriptsize$\pm$6.29} & 2.98{\scriptsize$\pm$0.41} & 26.10{\scriptsize$\pm$2.89} & 38.58{\scriptsize$\pm$2.59} & 4.41{\scriptsize$\pm$1.24} & 24.76{\scriptsize$\pm$1.31} & 34.57{\scriptsize$\pm$0.84} & 5.65{\scriptsize$\pm$1.20} & 23\\
    & TransE   & 32.17{\scriptsize$\pm$6.52} & 57.00{\scriptsize$\pm$7.06} & 3.97{\scriptsize$\pm$0.90} & 38.58{\scriptsize$\pm$3.64} & 50.04{\scriptsize$\pm$2.63} & 7.74{\scriptsize$\pm$1.71} & 21.27{\scriptsize$\pm$1.81} & 30.22{\scriptsize$\pm$1.96} & 5.11{\scriptsize$\pm$1.16} & 22\\
    & ComplEx  & 62.53{\scriptsize$\pm$1.71} & 73.59{\scriptsize$\pm$0.90} & 17.82{\scriptsize$\pm$2.58} & 38.14{\scriptsize$\pm$3.55} & 53.87{\scriptsize$\pm$3.09} & 6.50{\scriptsize$\pm$2.48} & 22.33{\scriptsize$\pm$1.05} & 32.17{\scriptsize$\pm$2.38} & 5.16{\scriptsize$\pm$1.07} & 19 \\
    & DistMult & 64.90{\scriptsize$\pm$1.59} & 73.12{\scriptsize$\pm$0.28} & 19.28{\scriptsize$\pm$2.11} & 39.09{\scriptsize$\pm$3.93} & 54.82{\scriptsize$\pm$2.95} & 6.90{\scriptsize$\pm$2.50} & 21.32{\scriptsize$\pm$0.45} & 32.37{\scriptsize$\pm$1.62} & 5.02{\scriptsize$\pm$1.24} & 17\\
    \midrule
    \multirow{3}{*}{Heuristic}  
    & CN & 50.36{\scriptsize$\pm$0.03} & 50.36{\scriptsize$\pm$0.03} & 32.88{\scriptsize$\pm$0.09} &  33.32{\scriptsize$\pm$0.02} & 33.32{\scriptsize$\pm$0.02} & 21.13{\scriptsize$\pm$0.02} & 27.20{\scriptsize$\pm$0.01} & 27.20{\scriptsize$\pm$0.01} & 14.66{\scriptsize$\pm$0.06} & 21\\
    & AA & 50.36{\scriptsize$\pm$0.03} & 50.36{\scriptsize$\pm$0.03} & 47.33{\scriptsize$\pm$0.09} & 33.32{\scriptsize$\pm$0.02} & 33.32{\scriptsize$\pm$0.02} & 24.61{\scriptsize$\pm$0.11} & 27.20{\scriptsize$\pm$0.01} & 27.20{\scriptsize$\pm$0.01} & 19.87{\scriptsize$\pm$0.30} & 18\\
    & RA & 50.36{\scriptsize$\pm$0.03} & 50.36{\scriptsize$\pm$0.03} & 47.17{\scriptsize$\pm$0.11} & 33.32{\scriptsize$\pm$0.02} & 33.32{\scriptsize$\pm$0.02} & 23.94{\scriptsize$\pm$0.16} & 27.20{\scriptsize$\pm$0.01} & 27.20{\scriptsize$\pm$0.01} & 19.16{\scriptsize$\pm$0.27} & 20\\
    \midrule          
    \multirow{4}{*}{Structure}  
    & PPR/sym & 84.74{\scriptsize$\pm$0.00} & 88.93{\scriptsize$\pm$0.00} & \cellcolor{Emerald!20}58.86{\scriptsize$\pm$0.98} & 69.81{\scriptsize$\pm$0.02} & 72.95{\scriptsize$\pm$0.01} & 28.04{\scriptsize$\pm$0.91} & 65.68{\scriptsize$\pm$0.02} & 67.86{\scriptsize$\pm$ 0.02} & 26.57{\scriptsize$\pm$0.82} & 6 \\
    & Katz & 69.25{\scriptsize$\pm$0.02} & 69.25{\scriptsize$\pm$0.02} & 38.17{\scriptsize$\pm$0.12} & 66.02{\scriptsize$\pm$0.02} & 66.02{\scriptsize$\pm$0.02} & 30.94{\scriptsize$\pm$0.08} & 55.39{\scriptsize$\pm$0.01} & 55.39{\scriptsize$\pm$ 0.01} & 21.76{\scriptsize$\pm$0.21} & 14\\
    & DeepWalk & 84.00{\scriptsize$\pm$0.07} &  89.31{\scriptsize$\pm$0.03} & 44.39{\scriptsize$\pm$0.96} & 64.35{\scriptsize$\pm$0.01} & 71.78{\scriptsize$\pm$0.01} & 19.66{\scriptsize$\pm$1.00} & 21.37{\scriptsize$\pm$0.11} & 33.80{\scriptsize$\pm$ 0.10} & 4.47{\scriptsize$\pm$0.03} & 16\\
    & Node2Vec & 83.08{\scriptsize$\pm$0.05} & 88.38{\scriptsize$\pm$0.03} & 39.94{\scriptsize$\pm$1.10} & 63.95{\scriptsize$\pm$0.02} & 70.95{\scriptsize$\pm$0.01} & 20.68{\scriptsize$\pm$0.24} & 30.19{\scriptsize$\pm$0.18} & 40.81{\scriptsize$\pm$ 0.20} & 5.60{\scriptsize$\pm$0.03} & 15 \\
    \midrule 

    \multirow{4}{*}{GCNs}      
    & GCN & 91.46{\scriptsize$\pm$2.36} & 96.20{\scriptsize$\pm$1.71} & 45.84{\scriptsize$\pm$8.40} & 83.11{\scriptsize$\pm$2.19} & 89.59{\scriptsize$\pm$1.73} & 24.55{\scriptsize$\pm$4.02} & 45.07{\scriptsize$\pm$0.87} & 52.46{\scriptsize$\pm$ 1.44} & 17.62{\scriptsize$\pm$3.34} & 10 \\
    & GAT & 89.80{\scriptsize$\pm$2.00} & 95.34{\scriptsize$\pm$1.61} & 49.82{\scriptsize$\pm$10.04} & 74.23{\scriptsize$\pm$2.54} & 84.03{\scriptsize$\pm$1.59} & 18.13{\scriptsize$\pm$5.81} & 43.09{\scriptsize$\pm$1.22} & 53.49{\scriptsize$\pm$ 1.06} & 13.58{\scriptsize$\pm$4.33} & 13\\
    & SAGE & 86.40{\scriptsize$\pm$3.73} & 95.10{\scriptsize$\pm$1.57} & 46.03{\scriptsize$\pm$6.70} & 86.55{\scriptsize$\pm$0.55} & 93.36{\scriptsize$\pm$0.70} & 31.66{\scriptsize$\pm$5.33}  & 45.42{\scriptsize$\pm$3.12} & 56.69{\scriptsize$\pm$ 3.13} & 11.52{\scriptsize$\pm$1.67} & 8 \\
    & GIN & 91.54{\scriptsize$\pm$2.91} & 96.05{\scriptsize$\pm$2.13}  & 51.90{\scriptsize$\pm$6.65} & \cellcolor{Emerald!20}86.92{\scriptsize$\pm$1.68} & 92.02{\scriptsize$\pm$0.91} & 24.63{\scriptsize$\pm$2.24} & 45.35{\scriptsize$\pm$2.58} & 53.22{\scriptsize$\pm$ 2.05} & 14.79{\scriptsize$\pm$4.53} & 9\\
    \midrule      
    \multirow{7}{*}{GCN4LP}            
    & SEAL & 87.38{\scriptsize$\pm$3.06} & 92.03{\scriptsize$\pm$2.96} & 37.81{\scriptsize$\pm$9.93} & 84.62{\scriptsize$\pm$3.53} & 89.52{\scriptsize$\pm$1.27} & \cellcolor{Emerald!40}49.02{\scriptsize$\pm$13.91} & 56.98{\scriptsize$\pm$1.89} &  67.34{\scriptsize$\pm$ 3.74}  & 22.47{\scriptsize$\pm$3.69} & 5\\
    & NeoGNN & 81.03{\scriptsize$\pm$3.11} & 90.04{\scriptsize$\pm$2.02} & 41.48{\scriptsize$\pm$5.11} & 73.17{\scriptsize$\pm$5.29} & 81.25{\scriptsize$\pm$8.14} & 31.44{\scriptsize$\pm$3.85} & 64.54{\scriptsize$\pm$11.14} & 69.34{\scriptsize$\pm$ 8.56}  & 28.07{\scriptsize$\pm$15.62} & 7 \\
    & ELPH & 87.30{\scriptsize$\pm$4.94} & 94.91{\scriptsize$\pm$2.17} & 39.86{\scriptsize$\pm$10.20} & 59.19{\scriptsize$\pm$5.58} & 74.62{\scriptsize$\pm$1.64} & 24.61{\scriptsize$\pm$3.17} & 57.66{\scriptsize$\pm$1.55} & 66.95{\scriptsize$\pm$ 3.62} & 29.22{\scriptsize$\pm$5.95} & 11\\
    & BUDDY & 87.82{\scriptsize$\pm$3.14} & 95.42{\scriptsize$\pm$2.26} & 30.78{\scriptsize$\pm$5.55} & 76.14{\scriptsize$\pm$3.46} & 89.25{\scriptsize$\pm$2.27} & 19.46{\scriptsize$\pm$2.42} & 52.25{\scriptsize$\pm$2.01} & 60.49{\scriptsize$\pm$ 0.94} & 18.75{\scriptsize$\pm$3.71} & 12\\
    & HL-GNN &
    90.59{\scriptsize$\pm$3.41} & 94.62{\scriptsize$\pm$1.87} & 50.35{\scriptsize$\pm$10.07} & 85.14{\scriptsize$\pm$1.83} & 91.87{\scriptsize$\pm$1.36} & 31.49{\scriptsize$\pm$7.84} & 76.51{\scriptsize$\pm$1.68} & 84.23{\scriptsize$\pm$0.82} & 24.21{\scriptsize$\pm$7.69} & 4\\
    & NCN & \cellcolor{Emerald!20}96.16{\scriptsize$\pm$1.62} & 
    \cellcolor{Emerald!20}98.74{\scriptsize$\pm$0.96} & 45.76{\scriptsize$\pm$16.39} & 86.44{\scriptsize$\pm$2.03} & 93.21{\scriptsize$\pm$1.10} & 25.92{\scriptsize$\pm$4.33} & \cellcolor{Emerald!20}82.34{\scriptsize$\pm$2.45} & 88.83{\scriptsize$\pm$ 1.43} & \cellcolor{Emerald!20}37.92{\scriptsize$\pm$13.21} & 2 \\
    & NCNC & 
    95.42{\scriptsize$\pm$2.41} & 98.67{\scriptsize$\pm$0.76}  & 48.68{\scriptsize$\pm$18.60} & 86.49{\scriptsize$\pm$0.99} &  \cellcolor{Emerald!20}93.74{\scriptsize$\pm$0.25} & 20.31{\scriptsize$\pm$6.51} & 81.86{\scriptsize$\pm$1.64} & \cellcolor{Emerald!20}89.13{\scriptsize$\pm$ 2.08} & 35.67{\scriptsize$\pm$12.30}& 3 \\
    \midrule                     
    Ours & mpnet-\method &  \cellcolor{Emerald!40}99.92{\scriptsize$\pm$0.18} & \cellcolor{Emerald!40} 99.92{\scriptsize$\pm$0.18}  & \cellcolor{Emerald!40} 68.43{\scriptsize$\pm$14.23} & \cellcolor{Emerald!40} 91.67{\scriptsize$\pm$4.96} & \cellcolor{Emerald!40} 97.13{\scriptsize$\pm$ 1.74} & \cellcolor{Emerald!20}35.63{\scriptsize$\pm$5.75} & \cellcolor{Emerald!40}89.17{\scriptsize$\pm$5.45} & \cellcolor{Emerald!40}94.85{\scriptsize$\pm$3.15} & \cellcolor{Emerald!40}45.70{\scriptsize$\pm$3.88} & 1 \\
    \bottomrule
  \end{tabular}
  \end{adjustbox}
\end{table*}

\noindent \textbf{Metrics \& Dataset Split.}  We utilize structure-sensitive metrics including Hits@50, Hits@100 and MRR (Mean Reciprocal Rank). Hits@K quantifies the ratio of positive edges ranked within the top K positions, while MRR evaluates the model’s ability to rank positive above negative ones.  We apply a uniform random split of 80\%, 15\% and 5\% across all graphs. All metrics are reported as the mean and standard deviation, averaged over five random seeds. We exclude the target link (link to be predicted) in each mini-batch.

\subsubsection{Results.} In  \Cref{tab:compare-gcn}, \method significantly outperforms established baselines, including the competitive NCNC method, with improvements of up to 5.72\% on Arxiv 2023 (Hits@100). On Pubmed, \method achieves a 4.75\% improvement in Hits@50 over the strongest prior model, GIN. On Cora, our method outperforms the second-best PageRank by 9.57\% in MRR. Across nine comparative experiments, \method surpasses eight different state-of-the-art approaches.  These results affirm the effectiveness of fusing text modalities with GCN within an iterative architecture. Although other GCN-based methods also show a certain level of competitiveness (NCNC, NCN, PageRank, SEAL), \method consistently provides performance improvements across all datasets. This demonstrates the necessity of integrating language embeddings with graph structure.

\subsection{RQ2: Comparison with PLM-based Approaches}
\label{subsec: RQ2}

Next, we compare \method’s performance with PLM-based approaches on the link prediction task, where the objective is to retrieve relevant conceptual embedding that enhances structural learning in TAGs. 

\subsubsection{Empirical Setup for RQ2} 
\noindent \textbf{Datasets \& Metrics.} We use the same datasets, data splits and evaluation metrics as described in \Cref{subsec: RQ2}. \\
\noindent \textbf{Benchmarked Approaches.} We consider two common frameworks for PLMs: \textit{Zero-Shot} and \textit{Fine-Tuning}. In the \textit{Zero-Shot} setting, the PLM is directly applied to downstream tasks without additional training, relying on its pretrained knowledge to generalize across domains. In contrast, \textit{Fine-Tuning} adapts the PLM to a specific task through gradient-based optimization on labeled data. 

We group the methods into four categories based on their choice of decoder ($D^*$) and training framework ($T^*$) in \Cref{tab:compare-plm}: 
1) \uline{ZS-MLP} treats the PLM as a frozen feature extractor, using the generated text embeddings as input for binary link prediction. 
2) \uline{FT-MLP} fine-tunes the final encoder layers of the PLM jointly with the MLP \citep{chen2024exploring}. 
3) \uline{ZS-GCN} leverages PLM-generated text embeddings as node features for a GCN, while keeping the PLM parameters frozen during training \citep{duan2024simteg, li2024tegdb}. 
4) \uline{FT-GCN} jointly optimizes both the PLM and GCN parameters within a cross-modal framework, following approaches similar to \citep{Xue2023EfficientEL, Zhu2024EfficientTA}.

\noindent \textbf{Choices and Implementation of PLM.} Our study incorporates a diverse set of PLM architectures, including encoder-only models BERT \citep{Devlin2019BERTPO} and sentence transformers such as MPNet \citep{Song2020MPNetMA}, e5-large-v2 \citep{Wang2022TextEB} and MiniLM-L6-v2 \citep{Reimers2019SentenceBERTSE}. Additionally, we include Meta-LLaMA-3-8B \citep{Dubey2024TheL3} as a case study for large-scale text embedding. \modify{These five backbones are chosen to span the three dominant PLM architecture families (encoder-only, sentence-transformer, decoder-only) across two orders of magnitude in parameter count, rather than to exhaustively benchmark the latest embedding models; \Cref{subfig:_params} confirms performance scales consistently with size within each family, isolating architecture-family effects from raw scale.} For consistency and comparability with prior work, we also evaluate shallow embedding methods, including bag-of-words \citep{Harris1954DistributionalS} and Word2Vec \citep{mikolov2013efficientestimationwordrepresentations}. We leverage the [EOS] token in LLaMA3 and the [CLS] token in sentence embedding models (MPNet \citep{Song2020MPNetMA}, e5-large-v2 \citep{Wang2022TextEB} and MiniLM-L6-v2 \citep{Reimers2019SentenceBERTSE}) as node features. We froze the layer up to the 6-th encoder layer in MiniLM \citep{Reimers2019SentenceBERTSE}; up to the 32-nd layer in Llama3 \citep{Dubey2024TheL3};  up to the 12-th encoder layer in BERT and up to the 23rd encoder layer in e5-large \citep{Wang2022TextEB}. 

\noindent \textbf{Hyperparameter Tuning.} We utilize hierarchical grid search for learning rate, weight decay, feature dimensions and MLP layers for the model listed in  \Cref{tab:compare-plm}. We provide more details about hyperparameter tuning in the \Cref{app:hyper_tune}. 

\setlength{\tabcolsep}{3pt}
\begin{table*}[b]
  \centering
  \tiny
  \caption{\textbf{Comparison with PLM-agnostic approaches.} Results of \method (mean $\pm$ standard deviation) on Hits@50, Hits@100, and MRR across five random splits of the Cora, PubMed, and Arxiv\_2023 datasets. The top 1–2 models are highlighted in emerald. \textbf{T:} Training framework — Zero-Shot (ZS) or Fine-Tune (FT). \textbf{D:} Decoder choice — GCN or MLP. Additional results are provided in ~\Cref{sec:app-full-result}. }
  \label{tab:compare-plm}
  \footnotesize
  \arrayrulecolor{black}
  \begin{adjustbox}{width=\textwidth}
  \begin{tabular}{lllccccccccc c}
    \toprule
     \multirow{2}{*}{\textbf{D}*} & \multirow{2}{*}{\textbf{T}*} & \multirow{2}{*}{Models}
     & \multicolumn{3}{c}{Cora}
     & \multicolumn{3}{c}{PubMed}
     & \multicolumn{3}{c}{Arxiv 2023}
     & \multirow{2}{*}{Rank} \\
     \cmidrule(lr){4-6}\cmidrule(lr){7-9}\cmidrule(lr){10-12}
     & & & Hits@50 & Hits@100 & MRR & Hits@50 & Hits@100 & MRR & Hits@50 & Hits@100 & MRR & \\
    \midrule
    \multirow{9}{*}{MLP} & \multirow{4}{*}{ZS} & BERT        & 35.79{\scriptsize$\pm$2.50} & 56.90{\scriptsize$\pm$3.26} & 3.42{\scriptsize$\pm$0.47} & 36.12{\scriptsize$\pm$0.37} & 48.73{\scriptsize$\pm$ 1.43} & 6.56{\scriptsize$\pm$0.70} & 37.66{\scriptsize$\pm$1.57} & 48.74{\scriptsize$\pm$ 1.15} & 10.04{\scriptsize$\pm$0.85} & 17\\
                         &                      & MiniLM      & 83.39{\scriptsize$\pm$0.00} & 92.99{\scriptsize$\pm$0.00} & 34.29{\scriptsize$\pm$4.10} & 66.35{\scriptsize$\pm$0.29} & 81.90{\scriptsize$\pm$ 0.03}& 21.54{\scriptsize$\pm$0.11} & 68.15{\scriptsize$\pm$0.09} & 77.62{\scriptsize$\pm$ 0.03} & 16.91{\scriptsize$\pm$0.18} & 10 \\
                         &                      & e5-large-v2 & 64.35{\scriptsize$\pm$1.56} & 83.10{\scriptsize$\pm$0.80}  & 24.40{\scriptsize$\pm$2.48} & 71.32{\scriptsize$\pm$0.86} & 82.59{\scriptsize$\pm$ 0.26}  & 21.79{\scriptsize$\pm$1.58} & 75.03{\scriptsize$\pm$0.28} &84.09{\scriptsize$\pm$ 0.24}  & 21.69{\scriptsize$\pm$0.03} & 11\\
                         &                      & Llama-3-8B  & 89.15{\scriptsize$\pm$0.72} & 95.64{\scriptsize$\pm$0.41} & 31.19{\scriptsize$\pm$3.49} & 79.87{\scriptsize$\pm$1.19} & 89.01{\scriptsize$\pm$ 0.53} & 22.87{\scriptsize$\pm$4.47} & 83.18{\scriptsize$\pm$1.19} & 89.91{\scriptsize$\pm$ 0.19} & 22.85{\scriptsize$\pm$1.12} & 7 \\
     \cmidrule(lr){2-12}
                         & \multirow{4}{*}{FT} & BERT        & 89.17{\scriptsize$\pm$2.86} & 96.99{\scriptsize$\pm$1.36} & 30.90{\scriptsize$\pm$4.33} & 73.70{\scriptsize$\pm$4.01} & 84.45{\scriptsize$\pm$ 2.92} & 17.11{\scriptsize$\pm$3.90} & 77.75{\scriptsize$\pm$3.46} & 87.56{\scriptsize$\pm$ 2.05} & 29.54{\scriptsize$\pm$3.98} & 8 \\
                         &                      & e5-large    & 92.09{\scriptsize$\pm$1.70} & 96.92{\scriptsize$\pm$1.35} & 38.63{\scriptsize$\pm$9.39} & 76.26{\scriptsize$\pm$2.55} & 87.23{\scriptsize$\pm$ 1.60} & 19.75{\scriptsize$\pm$5.81} & 80.48{\scriptsize$\pm$2.52} & 89.35{\scriptsize$\pm$1.33} & \cellcolor{Emerald!20}31.73{\scriptsize$\pm$6.62} & 4 \\
                         &                      & MiniLM      & 92.49{\scriptsize$\pm$2.13} & 96.68{\scriptsize$\pm$1.69} & 35.55{\scriptsize$\pm$5.82} & 75.87{\scriptsize$\pm$3.72} & 86.80{\scriptsize$\pm$ 1.98} & 20.79{\scriptsize$\pm$6.32} & 80.20{\scriptsize$\pm$2.62} & 88.38{\scriptsize$\pm$1.06} & 29.86{\scriptsize$\pm$5.82} & 6 \\
                         &                      & mpnet       & 93.44{\scriptsize$\pm$1.64} & 97.78{\scriptsize$\pm$0.66} & 22.55{\scriptsize$\pm$10.71} & 63.27{\scriptsize$\pm$31.76} & 90.69{\scriptsize$\pm$2.49} & 9.38{\scriptsize$\pm$3.12} &  82.72{\scriptsize$\pm$1.28} &  91.44{\scriptsize$\pm$0.75} & 8.42{\scriptsize$\pm$6.49} & 9 \\
    \midrule
    \multirow{8}{*}{GCN} & \multirow{4}{*}{ZS} & BERT-BUDDY  & 63.32{\scriptsize$\pm$5.74} & 81.42{\scriptsize$\pm$2.86} & 14.64{\scriptsize$\pm$3.70} & 62.03{\scriptsize$\pm$3.16} & 78.54{\scriptsize$\pm$1.26} & 13.15{\scriptsize$\pm$2.84} & 35.19{\scriptsize$\pm$1.54} & 43.30{\scriptsize$\pm$1.65} & 10.20{\scriptsize$\pm$2.88} & 15 \\
                         &                      & Llama-BUDDY & 91.23{\scriptsize$\pm$1.66} & 97.47{\scriptsize$\pm$1.44} & 29.10{\scriptsize$\pm$5.72} & 88.10{\scriptsize$\pm$2.08} & 94.60{\scriptsize$\pm$0.98} & 24.57{\scriptsize$\pm$3.66} & 75.15{\scriptsize$\pm$3.55} & 84.93{\scriptsize$\pm$1.68} & 28.74{\scriptsize$\pm$5.29} & 3 \\
                         &                      & MiniLM-NCNC &
    \cellcolor{Emerald!20}96.13{\scriptsize$\pm$1.20} & \cellcolor{Emerald!20}98.81{\scriptsize$\pm$0.49} & \cellcolor{Emerald!20}38.96{\scriptsize$\pm$13.20} & \cellcolor{Emerald!20}90.32{\scriptsize$\pm$1.52} & \cellcolor{Emerald!20}96.11{\scriptsize$\pm$ 0.60} & 22.56{\scriptsize$\pm$3.30} & 65.65{\scriptsize$\pm$1.80} & 70.61{\scriptsize$\pm$2.24} & 29.10{\scriptsize$\pm$3.83} & 5 \\
                         &                      & Llama-NCNC  &
    95.57{\scriptsize$\pm$1.02} & 98.73{\scriptsize$\pm$0.65} & 27.45{\scriptsize$\pm$7.86} & 84.65{\scriptsize$\pm$1.95} & 92.39{\scriptsize$\pm$ 1.46} & 20.51{\scriptsize$\pm$9.80} &
    \cellcolor{Emerald!20}84.68{\scriptsize$\pm$1.72} & \cellcolor{Emerald!20}91.90{\scriptsize$\pm$1.33} & 27.16{\scriptsize$\pm$11.48} & 2 \\
     \cmidrule(lr){2-12}
                         & \multirow{4}{*}{FT} & BERT-GIN    & 84.23{\scriptsize$\pm$4.08} & 86.99{\scriptsize$\pm$6.58} & 33.74{\scriptsize$\pm$4.12} & 29.44{\scriptsize$\pm$2.84} & 43.75{\scriptsize$\pm$5.46} & 24.26{\scriptsize$\pm$1.75} & 13.76{\scriptsize$\pm$2.22} & 25.18{\scriptsize$\pm$4.17} & 12.62{\scriptsize$\pm$0.55} & 16 \\
                         &                      & MiniLM-GIN  & 93.04{\scriptsize$\pm$3.54} & 92.01{\scriptsize$\pm$3.28} & 25.09{\scriptsize$\pm$0.47} & 45.31{\scriptsize$\pm$3.83} & 63.18{\scriptsize$\pm$1.34} & \cellcolor{Emerald!20}26.70{\scriptsize$\pm$3.61} & 49.11{\scriptsize$\pm$4.22} & 66.06{\scriptsize$\pm$2.60} & 21.48{\scriptsize$\pm$5.10} & 14 \\
                         &                      & mpnet-GIN   & 92.13{\scriptsize$\pm$4.12} & 97.55{\scriptsize$\pm$1.86} & 29.06{\scriptsize$\pm$7.96} & 46.82{\scriptsize$\pm$4.22} & 63.42{\scriptsize$\pm$2.77} & 21.86{\scriptsize$\pm$3.68} & 55.11{\scriptsize$\pm$3.70} & 64.38{\scriptsize$\pm$3.22} & 28.88{\scriptsize$\pm$5.89} & 12 \\
                         &                      & mpnet-SAGE  & 82.01{\scriptsize$\pm$4.19} & 93.88{\scriptsize$\pm$0.28} & 25.34{\scriptsize$\pm$8.06} & 57.58{\scriptsize$\pm$3.78} & 71.62{\scriptsize$\pm$2.78} & 21.91{\scriptsize$\pm$3.58} & 52.97{\scriptsize$\pm$5.05} & 66.28{\scriptsize$\pm$5.98} & 24.51{\scriptsize$\pm$4.54} & 13 \\
    \midrule
                         & Ours                & mpnet-\method &
      \cellcolor{Emerald!40}99.92{\scriptsize$\pm$0.18} &  \cellcolor{Emerald!40}99.92{\scriptsize$\pm$0.18}  &  \cellcolor{Emerald!40}68.43{\scriptsize$\pm$14.23} &  \cellcolor{Emerald!40}91.67{\scriptsize$\pm$4.96} &  \cellcolor{Emerald!40}97.13{\scriptsize$\pm$ 1.74} & \cellcolor{Emerald!40}31.66{\scriptsize$\pm$5.33} & \cellcolor{Emerald!40}89.17{\scriptsize$\pm$5.45} & \cellcolor{Emerald!40}94.85{\scriptsize$\pm$3.15} & \cellcolor{Emerald!40}45.70{\scriptsize$\pm$3.88} & 1 \\
    \bottomrule
  \end{tabular}
  \end{adjustbox}
\end{table*}

\subsubsection{Results} \method consistently outperforms all benchmarked approaches. On Cora, under the MRR metric, \method outperforms the best baseline, MiniLM-NCNC, by up to 43.06\%. For the Hits@K metric, Mpnet-\method surpasses the strongest baseline, ZS-GCN, by a maximum margin of 2\%. Furthermore, our \method consistently achieves state-of-the-art results, reaching nearly perfect Hits@50/100 on Cora and PubMed, and achieving substantially higher MRR on the large-scale Arxiv 2023 benchmark. We observe that existing zero-shot PLM-MLP baselines have lower MRR but high Hits@K across the three datasets, indicating that they fail to capture crucial structural properties. However, after fine-tuning, FT-MLP generally performs better, with BERT reaching 54.17\% on Cora. 
We also find that combining PLMs with GCNs (FT-GCN) does not yield stable or significant improvements over FT-MLP. This finding aligns with other benchmarks \citep{Chen2024TextspaceGF}, suggesting that simply combining PLMs and GNNs does not necessarily lead to improved performance. 
Our method, enhanced by three robust designs, underscores the effectiveness of integrating semantic understanding with structural reasoning, proving to be both principled and scalable for text-attributed link prediction.

\noindent \textbf{Impact of the LLM Scale.} Within the Zero-Shot - MLP category, we examine the effects of scaling the text encoder, from BERT embeddings to larger decoder-only models such as LLaMA-3-8B. Our findings reveal a clear correlation between PLM size and link prediction performance. Specifically, increasing the scale of the PLM backbone leads to a consistent improvement in average rank, advancing from 17th to seventh place. To further generalize this observation, we extend our evaluation across broader model categories, as shown in  \Cref{subfig:_params}. The observed trend remains consistent within specific model families, including E5~\citep{Wang2022TextEB}, BERT~\citep{Devlin2019BERTPO} and Sentence-Transformers~\citep{Reimers2019SentenceBERTSE, Song2020MPNetMA}. These results provide strong empirical evidence that the expressive capacity of PLMs directly influences link prediction performance, reinforcing the critical role of language modeling in enhancing structural learning.

\subsection{RQ3: Evaluation on Extended-Scale PLMs}
Building on the strong results obtained on three commonly used link prediction datasets, we further evaluate \method on a scaled dataset. The objective of this experiment is to assess \method’s ability to generalize beyond small-scale graphs and demonstrate its scalability for industrial-scale link prediction embedding tasks.

\subsubsection{Empirical Setup for RQ3}
\noindent \textbf{Dataset Setting.} 
\label{sec: proposed dataset} 
We collect nine datasets, including two generated TAGs from the PaperswithCode API \citep{Saier2023CoConAD} pwc small and pwc medium for link prediction. It has an expanded scale with a continuous spectrum of node sizes ranging from $10^2$ to $10^6$. We include Citationv8 \citep{yan2023a} and pwc medium to enable the study of large-scale settings. For all experiments, the results are reported on randomly sampled test edges for datasets larger than pwc\_medium in   \Cref{tab:llm_result}. We utilize metrics AUC due to its desirable property of \textbf{scale invariance}.  AUC assesses the model's ability to score positives higher than negatives, offering numerical stability and scale invariance.  We apply a uniform random split of 80\%, 15\% and 5\% across all graphs. All results are averaged over five random seeds. We exclude the target link (link to be predicted) in each mini-batch.

\noindent \textbf{Hyperparameter Ranges.} Due to the limited computation budget, we use the tuned parameters of GCN and PLM from the \Cref{subsec: RQ2}. We only tune the optimizer-related parameters such as learning rate, batch size and weight decay to adapt to our GPU. Further details on parameter tuning and experimental settings are provided in ~\Cref{subsec:app-hyperparameter-gcn2lp}. 

\noindent \textbf{Benchmarked Approaches.} We compare \method{} against the strongest baselines summarized in each category, as reported in \Cref{tab:compare-gcn}, \Cref{tab:compare-plm} and \Cref{tab:llm_result}, namely Llama3-MLP, NCNC and MiniLM-NCN.

\begin{table*}[t]
  \caption{\textbf{Evaluation on large-scale graphs.} We compare \method against the strongest baseline in each category using AUC. Reported values are mean accuracy $\pm$ standard deviation across five random data splits. The top 1–2 models are highlighted in emerald. }
    \label{tab:llm_result}
    \centering
    \small
    \begin{adjustbox}{width=\textwidth}
    \begin{tabular}{lcccccccccc}
        \toprule
        & \multicolumn{3}{c}{\textsc{Small}} & \multicolumn{4}{c}{\textsc{Medium}} & \multicolumn{2}{c}{\textsc{Large}} & \multirow{2}{*}{Rank}  \\ 
        \cmidrule(lr){2-4}\cmidrule(lr){5-8}\cmidrule(lr){9-10}
        & PwC small & Cora & Arxiv 2023 & PubMed & PwC_\text{medium} & History & Photo & Ogbn-arxiv & Citationv8 & \\
        \cmidrule(lr){1-11}
        \multicolumn{7}{l}{\textbf{Embedding-MLP: Non-contextualized Shallow Embeddings}} \\
        TF-IDF & $63.50 \pm 8.59$ & $68.27 \pm 2.52$ & $76.65 \pm 1.95$ & $67.06 \pm 2.34$ & $70.94 \pm 0.94$ & $60.94 \pm 0.48$ & $62.80 \pm 0.62$ & $62.63 \pm 2.88$ & $57.18 \pm 1.19$ & \textbf{14} \\
        Word2Vec & $51.00 \pm 2.24$ & $60.15 \pm 1.77$ & $85.22 \pm 0.92$ & $83.88 \pm 0.91$ & $81.79 \pm 0.55$ & $65.54 \pm 0.21$ & $64.71 \pm 0.12$ & $85.57 \pm 0.28$ & $80.50 \pm 0.35$ & \textbf{13} \\ \midrule
        \multicolumn{7}{l}{\textbf{ZS-MLP: Local Sentence Embedding Models}} \\
        MiniLM-L6-v2 & $51.90 \pm 11.54$ & $91.22 \pm 0.04$ & $95.22 \pm 0.00$ & $96.20 \pm 0.00$ &  $98.39 \pm 0.04$ & $94.64 \pm 0.01$ & $84.14 \pm 0.03$ & $98.22 \pm 0.01$ & $97.86 \pm 0.01$ & \textbf{11} \\
        e5-large-v2 & $80.60 \pm 2.57$ & $83.87 \pm 0.23$ & $95.72 \pm 0.01$ & $96.73 \pm 0.03$ & $97.83 \pm 0.01$ & $95.97 \pm 0.00$ & $85.04 \pm 0.44$ & $97.92 \pm 0.01$ & $98.05 \pm 0.01$ & \textbf{10} \\
        BERT & $69.85 \pm 2.40$ & $65.09 \pm 1.41$ & $86.37 \pm 0.27$ & $88.96 \pm 0.31$ & $83.85 \pm 0.33$ & $90.26 \pm 0.36$ & $73.12 \pm 0.76$ & $86.89 \pm 0.26$ & $86.22 \pm 0.48$ & \textbf{12} \\ \midrule

        \multicolumn{7}{l}{\textbf{FT-MLP}} \\
        mpnet-FT & $85.93 \pm 5.86$ & $94.71 \pm 1.16$ & $97.36 \pm 0.33$ &  $98.06 \pm 0.19$ &  $97.72 \pm 0.54$ & $93.91 \pm 0.64$ & $82.26 \pm 0.92$ & $98.21 \pm 0.26$ &  $98.17 \pm 0.81$ & \textbf{8}\\ 
        e5-large-v2-FT &  $86.95 \pm 4.93$ & $94.27 \pm 0.85$ &  $97.39 \pm 0.33$ & $97.64 \pm 0.36$ & $94.06 \pm 0.84$ & $94.82 \pm 1.18$ & $86.26 \pm 1.16$ & $97.68 \pm 0.17$ & $97.08 \pm 1.42$ & \textbf{9}\\
        MiniLM-L6-v2-FT &  $87.09 \pm 2.51$ & $93.98 \pm 0.85$ & $97.25 \pm 0.36$ & $97.79 \pm 0.14$ & $97.95 \pm 0.44$ & $95.58 \pm 0.51$ & $88.34 \pm 0.75$ &  $98.80 \pm 0.16$ & $97.85 \pm 1.83$ & \textbf{7}\\ \midrule

        \multicolumn{7}{l}{\textbf{ZS-MLP}} \\
        Llama-3-8B & \cellcolor{Emerald!40} $94.65 \pm 1.23$ & $92.60 \pm 0.12$ &   $97.62 \pm 0.03$ &  $98.09 \pm 0.10$ & $97.74 \pm 0.03$ & $97.28 \pm 0.08$ & $88.62 \pm 0.26$ & $99.06 \pm 0.05$ & $99.05 \pm 0.01$ & \textbf{6} \\ \midrule

        \multicolumn{7}{l}{\textbf{Pairwise GCN baseline}} \\
        NCN & $86.65 \pm 5.37$ &  $96.66 \pm 1.14$ & $97.30 \pm 0.26$ & $98.66 \pm 0.18$ &  $98.46 \pm 0.19$ & $97.77 \pm 0.30$ & $96.58 \pm 0.2$ & $98.96 \pm 0.07$ &  $98.18 \pm 0.20$ & \textbf{5} \\
        NCNC & $86.87 \pm 7.99$ & $96.56 \pm 1.04$ &   $97.42 \pm 0.37$ & $98.66 \pm 0.12$ & 
 $98.45 \pm 0.21$ &  $97.79 \pm 0.25$ & $96.79 \pm 0.25$ & $98.93 \pm 0.13$ &   $98.68 \pm 0.06$ & \textbf{4} \\ \midrule

        \multicolumn{7}{l}{\textbf{ZS-GCN}} \\
         MiniLM-NCN &   $87.15 \pm 6.84$ &  $96.93 \pm 0.54$ & $86.69 \pm 0.28$ & \cellcolor{Emerald!20} $98.97 \pm 0.10$ &  $98.99 \pm 0.16$ &   $99.12 \pm 0.11$ &  $99.57 \pm 0.03$ &  $99.58 \pm 0.07$ & $98.17 \pm 0.42$  & \textbf{3} \\
         mpnet-NCN &    $88.31 \pm 4.99$ &  $96.72 \pm 0.67 $ &  $97.82 \pm 0.22$ & $97.24 \pm 0.20$ & \cellcolor{Emerald!20}  $99.03 \pm 0.18$ & \cellcolor{Emerald!20}  $99.50 \pm 0.13$ &  $99.56 \pm 0.04$ &  $99.52 \pm 0.07$ & $98.15 \pm 0.42$ & \textbf{2}\\
         \midrule 

        \multicolumn{7}{l}{\textbf{Ours}} \\
        mpnet-\method & \cellcolor{Emerald!20} $89.36 \pm 5.37$ & \cellcolor{Emerald!40}$98.78 \pm 1.02$ & \cellcolor{Emerald!40}$98.79 \pm 0.49$ & \cellcolor{Emerald!40}$99.34 \pm 0.22$ & \cellcolor{Emerald!40}$99.34 \pm 0.09$ & \cellcolor{Emerald!40}  $99.54 \pm 0.01$ & \cellcolor{Emerald!40} $99.63 \pm 0.02$  & \cellcolor{Emerald!40}$99.72 \pm 0.04$ & \cellcolor{Emerald!40}$99.76 \pm 0.09$ & \textbf{1} \\ 
        \bottomrule
    \end{tabular}
    \end{adjustbox}
\end{table*}

\subsubsection{Results}  As shown in \Cref{tab:llm_result}, \method significantly outperforms all baseline models by at least 4.72\%, achieving a test accuracy of 89.36\% and consistently exceeding 99\% accuracy on large-scale graphs, thereby establishing new state-of-the-art results. While models such as MiniLM-NCN and MiniLM-NCNC perform competitively by incorporating pairwise structural embeddings, they lack an efficient design to fuse deeper contextual representations for link prediction.

We further investigate the impact of language models (LMs) across different graph scales, reporting results for shallow embeddings, PLMs, graph baselines, and our proposed methods across all datasets. Shallow embeddings, namely TF-IDF and Word2Vec \citep{mikolov2013efficientestimationwordrepresentations}, exhibit the weakest performance, highlighting their inability to capture complex semantic information. Local sentence embedding models (e.g., MiniLM, e5-large-v2) and fine-tuned PLMs substantially improve performance, particularly on smaller datasets; however, their performance gains tend to saturate on medium and large graphs. While Llama-3-8B achieves competitive results in the zero-shot setting on smaller datasets, it lags behind graph-based approaches as the graph scale increases.

Pairwise GCN baselines (NCN, NCNC) remain robust, consistently surpassing most PLM-based models. Notably, hybrid models (e.g., ZS-GCN) further boost accuracy, demonstrating the clear benefit of integrating PLM-derived representations with GCN-based structural reasoning. Our proposed \method (LMGJoint) achieves the best overall performance, ranking first on average across all datasets. This confirms the effectiveness of unifying semantic signals from PLMs with structural signals from GCNs in a joint framework.

\subsection{RQ4: Ablations on Graph Signals, Language Model Choice and Scale}
\label{sec:ablation study}

Next, we present our ablation studies that aim to investigate the impact of the data signals, the LM choice and scale, as well as the number of GCN layers.

\subsubsection{Impact of the Graph Signals}
First, we analyze the necessity of individual components via the ablation study in our proposed method (MiniLM-\method) on the Cora, Pubmed and Arxiv\_2013 datasets. Overall, the results, presented in \Cref{tab:ablation_study}, highlight the impact of key components on overall performance. 

\begin{table}[t]
    \centering
    \tiny
    \caption{Effect of \textbf{three graph signals} on Hits@50, MRR and AUC for three graphs.
    \modify{S1: Structural Proximity (Thresholded Common Neighbor Indicator, TCN, Eq.~\ref{equ:C2}); S2: Semantic Proximity (PLM text embeddings); S3: Propagated Proximity (GCN).}}
    \begin{adjustbox}{width=\textwidth}
    \begin{tabular}{lcccccccccc}
        \toprule
        \multirow{2}{*}{Method} 
            & \multicolumn{3}{c}{Cora}  
            & \multicolumn{3}{c}{PubMed} 
            & \multicolumn{3}{c}{Arxiv 2013}
            & \multirow{2}{*}{Rank} \\
        \cmidrule(lr){2-4}\cmidrule(lr){5-7}\cmidrule(lr){8-10}
        & Hits@50 & MRR & AUC  
        & Hits@50 & MRR & AUC  
        & Hits@50 & MRR & AUC & \\
        \midrule
        \method w/o S1 
            & 50.98{\tiny$\pm$0.32} 
            & 48.52{\tiny$\pm$7.88} 
            & 75.27{\tiny$\pm$0.52}
            & 35.06{\tiny$\pm$0.08} 
            & 15.04{\tiny$\pm$2.88} 
            & 67.39{\tiny$\pm$0.34}
            & 55.11{\tiny$\pm$3.70} 
            & 18.88{\tiny$\pm$5.89} 
            & 92.63{\tiny$\pm$0.87} & 4 \\
        \method w/o S2 
            & \cellcolor{Emerald!20}95.26{\tiny$\pm$0.19} 
            & 39.17{\tiny$\pm$3.88} 
            & \cellcolor{Emerald!20}96.16{\tiny$\pm$0.04}
            & 72.16{\tiny$\pm$4.29} 
            & 14.12{\tiny$\pm$4.28} 
            & \cellcolor{Emerald!20}97.93{\tiny$\pm$0.20}
            & 65.65{\tiny$\pm$1.80} 
            & 29.10{\tiny$\pm$3.83} 
            & 91.75{\tiny$\pm$1.47} & 3 \\
        \method w/o S3 
            & 88.14{\tiny$\pm$0.13} 
            & 38.53{\tiny$\pm$5.79} 
            & 92.35{\tiny$\pm$0.23}
            & \cellcolor{Emerald!20}78.20{\tiny$\pm$1.33} 
            & \cellcolor{Emerald!20}23.36{\tiny$\pm$2.79} 
            & 97.36{\tiny$\pm$0.12}
            & \cellcolor{Emerald!20}86.47{\tiny$\pm$3.40} 
            & \cellcolor{Emerald!20}27.76{\tiny$\pm$7.53} 
            & 92.50{\tiny$\pm$3.28} & 2 \\
        \method        
            & \cellcolor{Emerald!40}96.44{\tiny$\pm$0.18} 
            & \cellcolor{Emerald!40}65.99{\tiny$\pm$4.38} 
            & \cellcolor{Emerald!40}97.38{\tiny$\pm$0.69}
            & \cellcolor{Emerald!40}78.84{\tiny$\pm$1.38} 
            & \cellcolor{Emerald!40}26.56{\tiny$\pm$5.38} 
            & \cellcolor{Emerald!40}98.26{\tiny$\pm$0.49}
            & \cellcolor{Emerald!40}90.61{\tiny$\pm$2.25} 
            & \cellcolor{Emerald!40}35.47{\tiny$\pm$10.91} 
            & \cellcolor{Emerald!40}99.05{\tiny$\pm$0.25} & 1 \\
        \bottomrule
    \end{tabular}
    \label{tab:ablation_study}
    \end{adjustbox}
\end{table}

\noindent $\bullet$ \textbf{Effect of \textcolor{royalgreen}{\Cone} (S1).}
\textcolor{royalgreen}{Removing \Cone from \method leads to a substantial performance drop, with AUC falling to 75.27\% on Cora and 67.39\% on Pubmed.
This sharp decline highlights the essential role of structural proximity in enabling the model to identify strongly connected node pairs via shared neighborhood information.
It also shows that \Cone alone is sufficient to improve AUC by effectively distinguishing structurally dissimilar samples.
Such randomly generated negative samples exhibit large contrasts in common-neighbor overlap and occur in considerable quantities.
Therefore, incorporating \Cone into the full model yields a significant performance improvement, particularly in terms of AUC.}
\begin{figure}[t!]
    \centering
    \begin{minipage}{0.8\linewidth}
        \centering
        \begin{subfigure}{\linewidth}
            \centering
            \includegraphics[width=\linewidth]{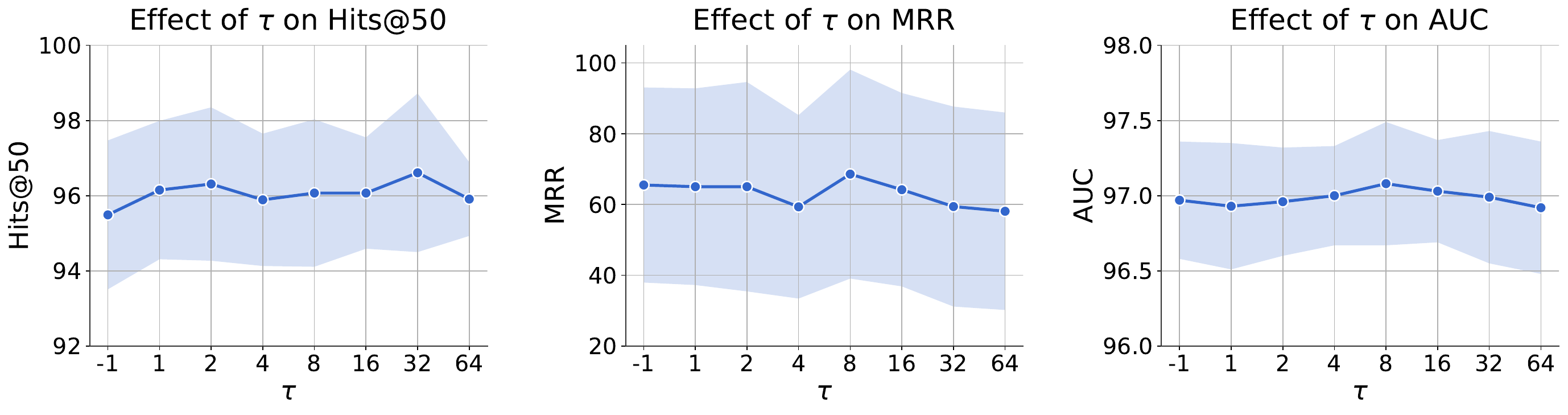}
            \caption{}
            \label{subfig:ablation_d}
        \end{subfigure}

        \vspace{0.6em}

        \begin{subfigure}{\linewidth}
            \centering
            \includegraphics[width=\linewidth]{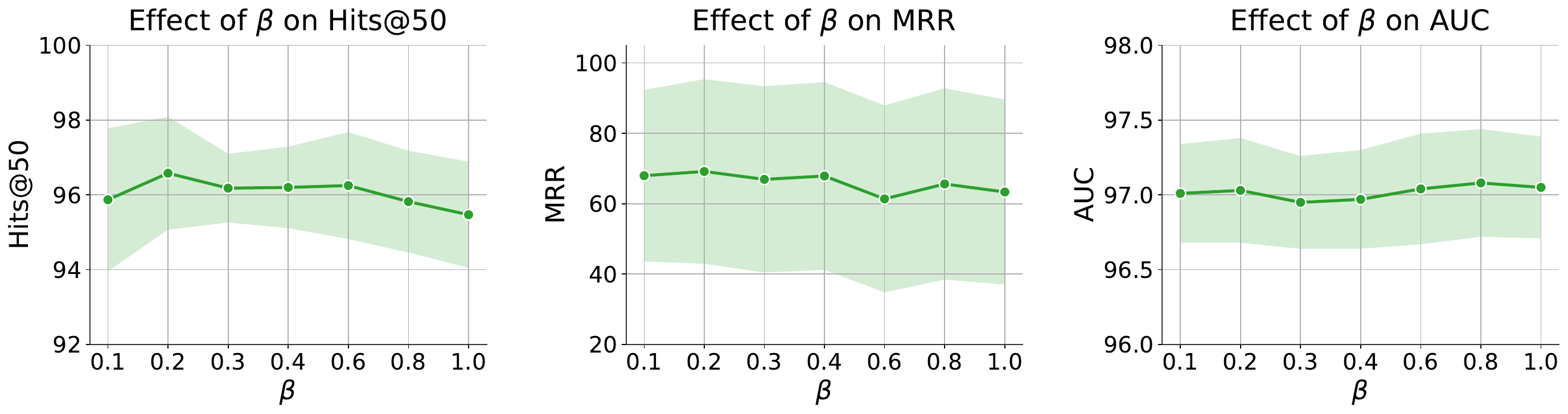}
            \caption{}
            \label{subfig:ablation_beta}
        \end{subfigure}
    \end{minipage}
    \textcolor{royalgreen}{\caption{(a) Sensitivity of \method{} to the common-neighbor degree cutoff $\tau$ used in the TCN threshold, on Hits@50, MRR, and AUC. (b) Sensitivity of \method{} to the structural-proximity mixing weight $\beta$ used in D2, on Hits@50, MRR, and AUC. Both sweeps are on Cora, averaged over five seeds; shaded bands show standard deviation.}}
    \label{fig:s1s2_sensitivity}
\end{figure}

\noindent $\bullet$ \textbf{Effect of \textcolor{royalgreen}{\Ctwo} (S2).}
\textcolor{royalgreen}{Eliminating \Ctwo leads to severe MRR drops, reaching 39.17\% on Cora and 14.12\% on Pubmed, as well as broader declines across other metrics. This sharp decline reinforces our earlier observation that semantic/textual features are critical for uncovering strong ties governed by shared topical or contextual meaning. By incorporating rich PLM-derived textual knowledge, \Ctwo complements structural similarity with insights into node content. This is especially beneficial for node pairs whose relatedness depends more on topic similarity than on adjacency-driven connectivity.}

\noindent \textbf{Additional Discussion of S2.} \textcolor{royalgreen}{We additionally study how \method{}'s performance depends on the common-neighbor degree cutoff $\tau$ used to threshold S1 (\Cref{equ:C2}) and the mixing weight $\beta$ applied to S1 when it is additively combined with S2 in D2. \Cref{fig:s1s2_sensitivity} reports Hits@50, MRR, and AUC on Cora across $\tau \in \{-1, 1, 2, 4, 8, 16, 32, 64\}$ ($\tau=-1$ denotes the unthresholded raw CN count) and $\beta \in \{0.1, 0.2, 0.3, 0.4, 0.6, 0.8, 1.0\}$, each averaged over five seeds. Hits@50 is remarkably stable across both sweeps (within 1.1 points), indicating that \method{} does not require careful tuning of either hyperparameter to remain competitive. MRR is more sensitive (up to 10.5 points across the $\tau$ sweep and 7.8 points across the $\beta$ sweep) but varies non-monotonically with both parameters. AUC is essentially insensitive to either hyperparameter, varying by at most 0.16 points across the $\tau$ sweep and 0.13 points across the $\beta$ sweep, consistent with its role as a coarser ranking metric relative to Hits@50 and MRR. In the $\tau$ sweep, Hits@50 peaks at $\tau=32$ while MRR peaks at $\tau=8$; in the $\beta$ sweep, both Hits@50 and MRR peak at $\beta=0.2$. \modify{In our main experiments we fix $\tau=2$ and $\beta=0.1$ for \emph{all} datasets; although the per-metric sweep optima above differ slightly, Hits@50 (our primary ranking metric) and AUC remain within roughly one point across the entire swept range, so this single global setting is sufficient while keeping the protocol simple and reproducible.}}

\noindent $\bullet$ \textbf{Effect of \Cthree (S3).} The impact of removing GCN (\Cthree) from \method varies by dataset. On Cora, omitting GCN drives Hits@50 down to 88.14\% and MRR to 38.53\%. Since Cora is homophilic (clustering coefficient: 0.51, edge homophily: 0.24), GCN plays a crucial role by fusing both semantic and structural information to enhance representation learning. On Pubmed, which has a low clustering coefficient (0.06), the drop is less visible, though MRR still declines by 3.2\%. These findings align with \Cref{tab:compare-gcn}, showing that networks with lower clustering reduce GCN's impact.
\modify{ Nevertheless, S3 remains consistently beneficial across all datasets in \Cref{tab:ablation_study}, every Hits@50 / MRR / AUC cell improves when S3 is included,  so we recommend it by default. A rigorous characterization of when propagated proximity helps most is difficult to establish empirically within this work, as few publicly available TAG benchmarks with rich text features exhibit strong heterophily; we point readers to \citet{Zhu2024OnTI} for a broader empirical study of how feature heterophily affects GNN-based link prediction.}

\begin{table}[t]
    \centering
    {\color{blue}
    \caption{Effect of \textbf{D1} on Cora and PubMed, Full (\method)
     matches the \method row of \Cref{tab:ablation_study}. }
    \begin{adjustbox}{max width=0.72\textwidth}
    \begin{tabular}{lcccccc}
        \toprule
        \multirow{2}{*}{Method}
            & \multicolumn{3}{c}{Cora}
            & \multicolumn{3}{c}{PubMed} \\
        \cmidrule(lr){2-4}\cmidrule(lr){5-7}
        & Hits@50 & MRR & AUC
        & Hits@50 & MRR & AUC \\
        \midrule
        Base
            & 72.64{\tiny$\pm$6.45}
            & 21.28{\tiny$\pm$6.74}
            & 84.96{\tiny$\pm$3.15}
            & 42.99{\tiny$\pm$3.70}
            & 8.02{\tiny$\pm$2.86}
            & 93.08{\tiny$\pm$0.78} \\
        Base + D1
            & 78.05{\tiny$\pm$1.93}
            & 21.25{\tiny$\pm$9.76}
            & 91.90{\tiny$\pm$1.25}
            & 42.85{\tiny$\pm$1.78}
            & 11.46{\tiny$\pm$4.92}
            & 95.32{\tiny$\pm$0.13} \\
        Base + D2
            & 92.13{\tiny$\pm$1.12}
            & 50.00{\tiny$\pm$3.45}
            & 92.02{\tiny$\pm$0.88}
            & 76.23{\tiny$\pm$2.11}
            & 25.75{\tiny$\pm$3.21}
            & 97.48{\tiny$\pm$0.32} \\  
        Full (\method)
            & 96.44{\tiny$\pm$0.18}
            & 65.99{\tiny$\pm$4.38}
            & 97.38{\tiny$\pm$0.69}
            & 78.84{\tiny$\pm$1.38}
            & 26.56{\tiny$\pm$5.38}
            & 98.26{\tiny$\pm$0.49} \\
        \bottomrule
    \end{tabular}
    \end{adjustbox}
    }
    \label{tab:design_ablation}
\end{table}

\subsubsection{Impact of the Design Components (D1--D3)} To assess the individual contributions of our designs and their effects on both
predictive performance and computational complexity, we conduct three ablation studies.
First, we compare the Base model, defined in \Cref{equ:final embed}, with Base+D1. We then compare the Base model
with Base+D2, i.e., adjacency matrix as feature embedding vs. thresholded common-neighbor embedding for link prediction.
\modify{Finally, since D3 is designed to reduce the memory consumption and training cost associated with the LLM
component, we compare memory usage and training time with and without D3 (Full vs. Full w/o D3), providing
empirical evidence that the cache-embedding strategy reduces computational complexity while maintaining
competitive performance.}

\noindent \textbf{Impact of Jump Connection (D1).}
\modify{
The results in \Cref{tab:design_ablation} demonstrate the contribution of the proposed
jump-connection design to link-prediction performance.} \modify{We show that D1 yields a modest, metric-dependent improvement over Base rather than a uniform gain:
AUC improves from 84.96\% to 91.90\% on Cora and from 93.08\% to 95.32\% on PubMed, while
Hits@50 moves only marginally (72.64\%$\to$78.05\% on Cora, 42.99\%$\to$42.85\% on PubMed)
and MRR is essentially unchanged on Cora (21.28\%$\to$21.25\%) though it improves on PubMed
(8.02\%$\to$11.46\%). These gains are consistent with D1's bypass connection strengthening
gradient flow to the PLM text embeddings, which most directly affects AUC's pairwise ranking,
while Hits@50 and MRR benefit more from the structural and compression refinements introduced
by D2 and D3.}

\noindent \textbf{Impact of Compressed Embedding (D2).}
\modify{\noindent As \Cref{tab:design_ablation} shows, this compression is not merely dimension-neutral: 
Base+D2 improves Hits@50 from 72.64\% to 92.13\% on Cora and from 42.99\% to 76.23\% on PubMed, with MRR 
improving similarly (21.28\%$\to$50.00\% and 8.02\%$\to$25.75\%, respectively). AUC improves too, but far 
less sharply (84.96\%$\to$92.02\% and 93.08\%$\to$97.48\%), consistent with AUC's known insensitivity once 
ranking quality is already high, Hits@50 and MRR are the more discriminative metrics for this comparison. 
This is because the thresholded common-neighbor indicator compresses the embeddings of node pairs with no 
shared neighbors to a zero vector, which not only reduces the feature dimensionality but also suppresses 
interference from these uninformative negative pairs, thereby improving MRR and Hits@50.}

\color{blue}
\begin{table}[t]
    \centering
    \caption{Training time and memory footprint of \method{} without D3 (\method{} w.o.\ D3)
     compared to Full \method{}, on Cora and PubMed. Training time is sec/epoch; Mem is total
     memory footprint (GB); vRAM is GPU memory footprint (GB). }
    \begin{adjustbox}{max width=\textwidth}
    \begin{tabular}{lcccccc}
        \toprule
        \multirow{2}{*}{Category}
            & \multicolumn{3}{c}{Cora}
            & \multicolumn{3}{c}{PubMed} \\
        \cmidrule(lr){2-4}\cmidrule(lr){5-7}
        & Training & Mem & vRAM
        & Training & Mem & vRAM \\
        \midrule
        \method{} w.o.\ D3
            & 17.4 & 57 & 40
            & 139.2 & 76 & 64 \\
        \method{}-Full
            & 4.7  & 25 & 25
            & 37.6 & 35 & 37 \\
        \midrule
        & Hits@50 & MRR & AUC
        & Hits@50 & MRR & AUC \\
        \cmidrule(lr){2-4}\cmidrule(lr){5-7}
        \method w.o.\ D3
            & 96.44{\tiny$\pm$0.18}
            & 65.99{\tiny$\pm$4.38}
            & 97.38{\tiny$\pm$0.69}
            & 98.84{\tiny$\pm$1.38}
            & 29.56{\tiny$\pm$5.38}
            & 99.06{\tiny$\pm$0.49} \\
        \method - Full 
            & 99.92{\scriptsize$\pm$0.18}  
            & 68.43{\scriptsize$\pm$4.23}
            & 98.78{\scriptsize$\pm$0.22}  
            & 97.13{\scriptsize$\pm$1.74} 
            & 31.66{\scriptsize$\pm$5.33} 
            & 99.09{\scriptsize$\pm$1.74} \\
        \bottomrule
    \end{tabular}
    \end{adjustbox}
    \label{tab:complexity_updated}
\end{table}
\color{black}
\noindent \textbf{Impact of Cache Embedding Strategy (D3).} \modify{Finally, we evaluate the impact of the cache-embedding strategy D3 by comparing the full model with its variant without D3 on Cora and PubMed. As shown in \Cref{tab:complexity_updated}, D3 substantially improves training efficiency: the full model is approximately $3.7\times$ faster on both datasets, reducing the per-epoch training time from $17.4$ to $4.7$ seconds on Cora and from $139.2$ to $37.6$ seconds on PubMed. It also reduces the total memory footprint by $56.1\%$ on Cora and $53.9\%$ on PubMed, while lowering GPU memory usage by $37.5\%$ and $42.2\%$, respectively. Importantly, these efficiency gains do not compromise predictive performance. The full model consistently matches or outperforms the variant without D3 across Hits@50, MRR, and AUC, including a slight degradation in Hits@50 on PubMed. These results demonstrate that D3 effectively removes the repeated computational and memory overhead associated with the language-model component while preserving, and in several cases improving, the quality of the learned representations.}

\subsubsection{Effect of GCN layers.} \Cref{subfig:_layer} depicts the test AUC and standard deviation (shaded regions) for different GCN depths (1 to 5 layers). We observe that in both datasets, test AUC exhibits slight fluctuations as the number of GCN layers increases from 1 to 5. For Pubmed, AUC remains relatively stable across different GCN depths, hovering around 98\%, with slight drops at four layers. These observations suggest that \method is generally robust to changes in GCN depth, tolerating moderate architectural adjustments without major performance losses.
\subsubsection{Impact of the LLM Choices and Scale}

\modify{
\method exhibits flexibility in adopting different LLMs as its backbone. To assess the impact of backbone choice, we compare three series of models: the E5 series \citep{Wang2022TextEB}, BERT \citep{Devlin2019BERTPO} and Sentence-Bert~\citep{Dubey2024TheL3}. As shown in ~\Cref{subfig:_params}, increasing the model size within each series consistently improves performance. This trend suggests that larger models, with more parameters, better capture complex relationships in multi-modal data. For instance, E5-large substantially outperforms E5-small and the performance of mpnet (a sentence transformer) is comparable to that of E5-large. While mpnet-based models generally yield higher MRR, MiniLM and E5-large variants demonstrate stronger stability on Hits@K. These findings indicate that the choice of LLM backbone is not only task-dependent but also metric-sensitive. }
\begin{figure}[h]
    \centering
    \begin{minipage}{0.95\linewidth}
        \centering
        \begin{subfigure}{0.48\linewidth}
            \centering
    \includegraphics[width=\linewidth]{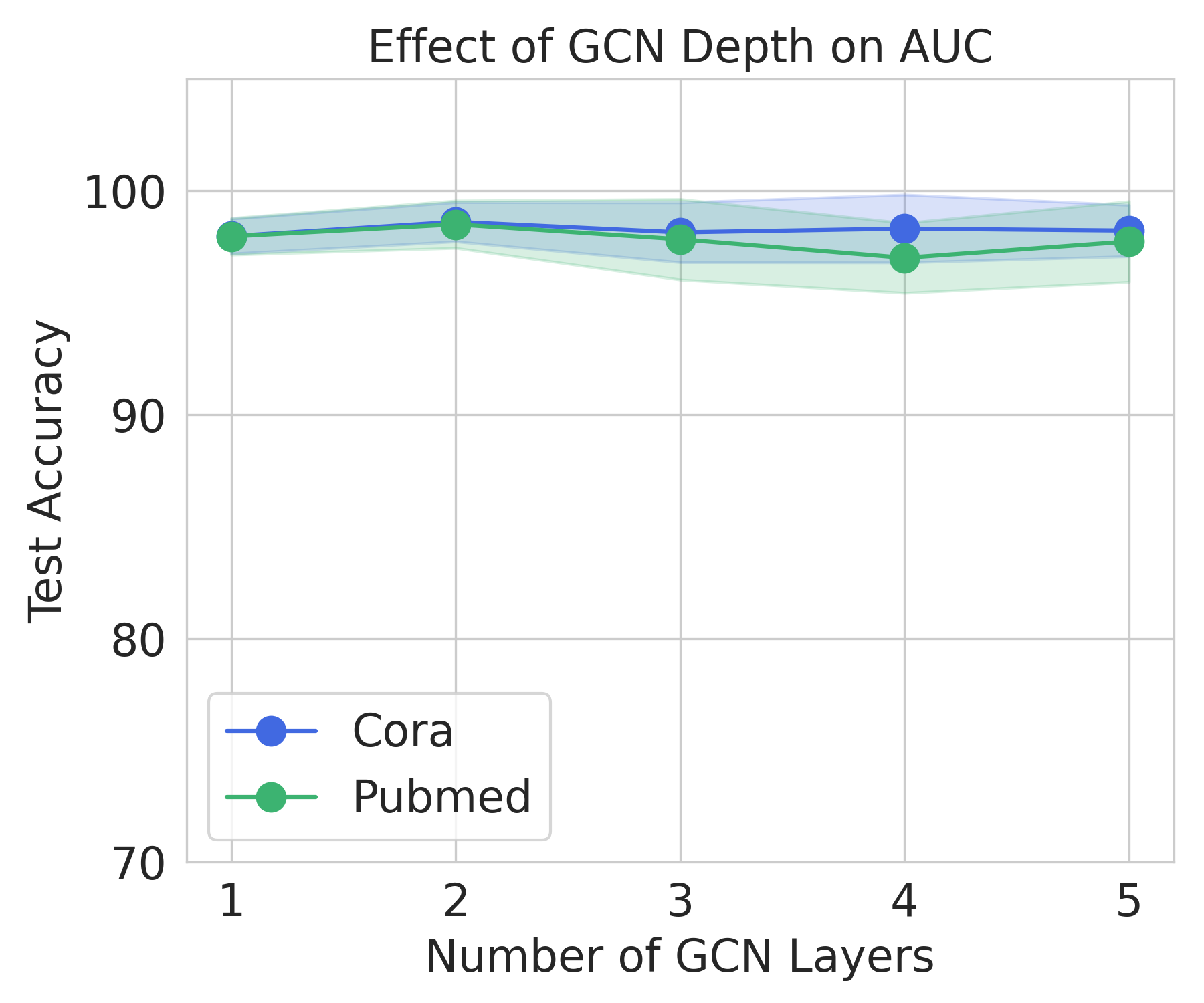}
            \caption{}
            \label{subfig:_layer}
        \end{subfigure}
        \hfill
        \begin{subfigure}{0.48\linewidth}
            \centering
            \includegraphics[width=\linewidth]{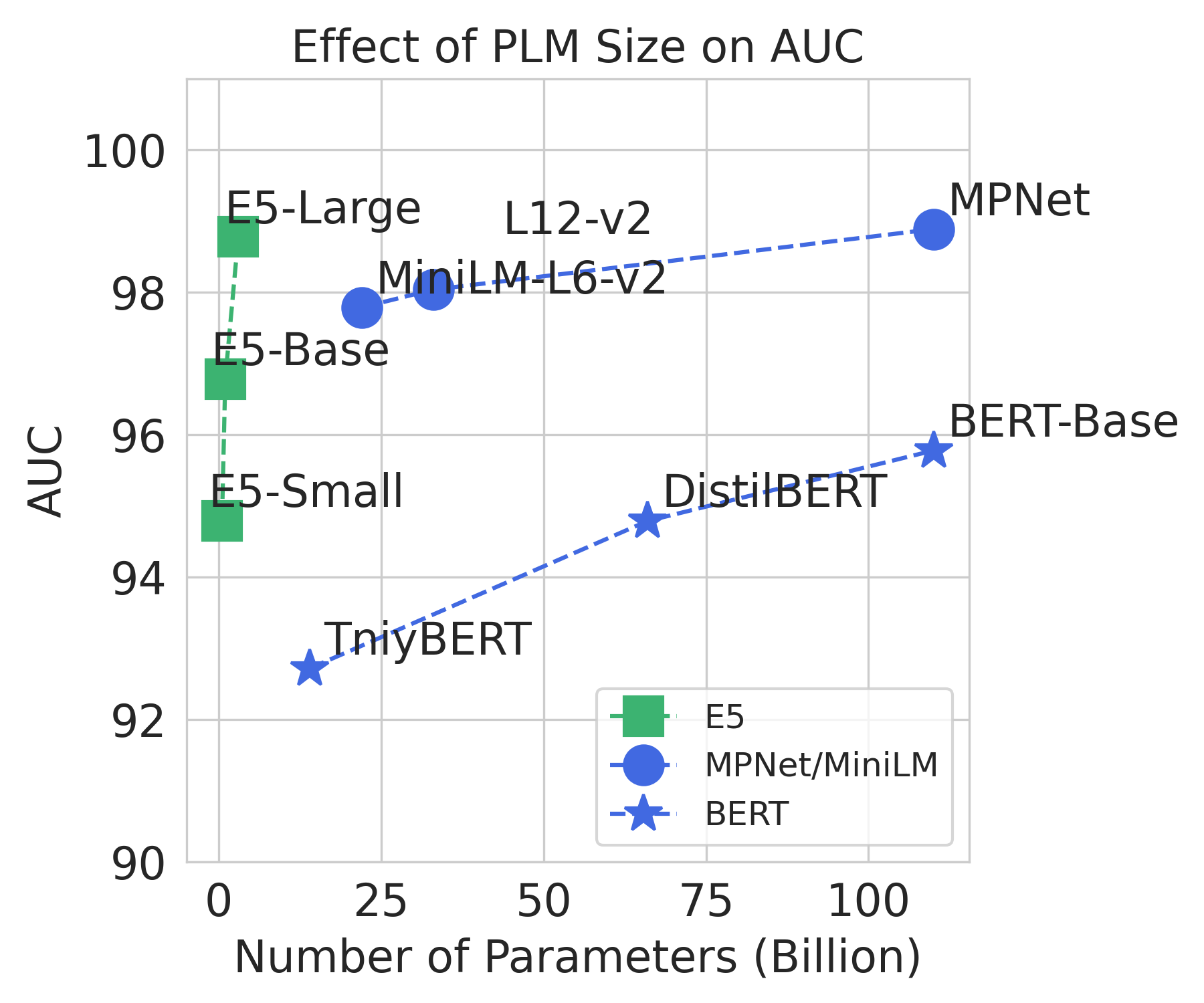}  
            \caption{}
            \label{subfig:_params}
        \end{subfigure}
    \end{minipage}
    \caption{(a) The effect of graph encoder depth (number of GNN layers) on test AUC for Cora and PubMed datasets. The shaded areas represent the standard deviation, indicating the variance in performance across different trials.  (b) The effect of LLM backbone choice on AUC for the Cora dataset. The figure shows three series: E5, BERT along with Sentence-Transformers. 
    }
    \label{fig:layers_plm}
\end{figure}

\section{Conclusion}
\label{sec: conclusion}
\color{black}
In this paper, we systematically addressed the challenge of leveraging textual and structural information for link prediction on TAGs. We propose \method, a lightweight yet powerful framework that integrates PLMs as learnable text encoders within an end-to-end GCN architecture.
It consists of necessary graph signals, spanning the predominant formalisms of links from latent graph theory. We concatenated these three signals into a baseline framework and then introduced three designs to further optimize gradient flow while reducing signal redundancy as well as the associated computational and memory costs. 
\method achieves significant performance gains. We evaluated the proposed method against 40 existing link prediction models across 9 diverse TAG datasets, ranging from small academic networks to billion-scale graphs. Empirically, our method outperforms both GCN- and PLM-based baselines by up to 20\% in MRR. Beyond link prediction, our approach opens new avenues for tackling other graph-related tasks, such as node classification and regression, where fully end-to-end methods can capitalize on the complementary strengths of language models and structural learners.
\color{black}

\textbf{Limitations and Future Work.} \method{} is designed for \emph{static} TAGs, and we make this scope explicit here. Extending it to dynamic TAGs, where node text and graph structure evolve over time \citep{Yu2023DyGLib, Zhang2024DTGB, Xu2025DynamicTAGMultiModal}, would require revisiting each of its three signals: (i) S1 Threshold Common Neighbor currently assumes a fixed adjacency matrix $\mathbf{A}$ and would need incremental maintenance of the thresholded common-neighbor indicator as edges arrive or expire; (ii) S3 Propagated proximity currently aggregates over a static $K$-hop neighborhood and would require a graph-time-aware aggregation scheme that respects the chronological order of edges, e.g., only propagating through edges that existed at or before the timestamp of the target link; (iii) by contrast, D3's cache-embedding strategy is the design most naturally compatible with a dynamic setting, since it already amortizes PLM encoding cost across training steps and would extend to streaming updates by re-encoding only nodes whose text changes rather than the full graph. We leave a full dynamic-TAG instantiation and evaluation of \method{} to future work.

\begin{acks}
This work was supported in part by the
\grantsponsor{nsf}{National Science Foundation}{}
under Grants No.~\grantnum{nsf}{IIS-2212143} and \grantnum{nsf}{IIS-2504090},
and in part by the
\grantsponsor{bmftr}{Federal Ministry of Research, Technology and Space (BMFTR)}{},
Germany, under award number~\grantnum{bmftr}{01IS23066}. 
Any opinions, findings, and conclusions or recommendations expressed in this material are those of the author(s) and do not necessarily reflect the views of the National Science Foundation.
\end{acks}
\color{black}
\bibliographystyle{ACM-Reference-Format}
\bibliography{sample-base}

\appendix
\newpage
\section{Proof of \Cref{theorem:d1}}
\label{app:d1}

\noindent\textbf{Influence score (setup, without D1).}
Adopted from \citep{Xu2018RepresentationLO}, the influence score \( I(x, y) \) of node \( x \) for any final embedding \( y \in \mathbb{R}^{|\mathcal{V}| \times d}\) is defined as the sum of the absolute values of the entries of the Jacobian matrix \( \frac{\partial h^{(k)}_x}{\partial h^{(0)}_y} \). We define the influence distribution \( I_x \) of \( x \in V \) by normalizing the influence scores:

\[
I_x(y) = \frac{I(x, y)}{\sum_{z} I(x, z)},
\]

or equivalently,
\[
I_x(y) = \frac{e^T \left| \frac{\partial h^{(k)}_x}{\partial h^{(0)}_y} \right| e}{\sum_{z \in V} e^T \left| \frac{\partial h^{(k)}_x}{\partial h^{(0)}_z} \right| e},
\]
where \( e \) is the all-ones vector. We calculate the influence score in \method without D1:

Denote by \( f^{(l)}_x \) the pre-activated feature of \( h^{(l)}_x \), i.e.,
\[
f^{(l)}_x = \frac{1}{\text{deg}_f(x)} \sum_{z \in \text{Ne}(x)} W_l h^{(l-1)}_z.
\]
For any \( l = 1, \dots, k \), we have
\[
\frac{\partial h^{(l)}_x}{\partial h^{(0)}_y} = \frac{1}{\text{deg}_f(x)} \cdot \text{diag}\left(\mathbf{1}_{f^{(l)}_x > 0}\right) \cdot W_l \cdot \sum_{z \in \text{Ne}(x)} \frac{\partial h^{(l-1)}_z}{\partial h^{(0)}_y}.
\]
By the chain rule, we obtain
\[
\frac{\partial h^{(k)}_x}{\partial h^{(0)}_y} = \sum_{\Psi} \left[ \frac{\partial h^{(k)}_x}{\partial h^{(0)}_y} \right]_p
= \sum_{\Psi} \prod_{l=1}^k \frac{1}{\text{deg}_f(v^l_p)} \cdot \text{diag}\left(\mathbf{1}_{f^{(l)}_{v^l_p} > 0}\right) \cdot W_l.
\]
Here, \( \Psi \) is the set of all random walks \( v^k_p, v^{k-1}_p, \dots, v^1_p, v^0_p \) of length \( k + 1 \) from node \( x \) to node \( y \). For any path \( p \), \( v^k_p = x \), \( v^0_p = y \), and for \( l = 1, \dots, k-1 \), \( v^{l-1}_p \in \text{Ne}(v^l_p) \).

For each path \( p \), an entry of the derivative can be expressed as
\[
\left[ \frac{\partial h^{(k)}_x}{\partial h^{(0)}_y} \right]^{(i,j)}_p = \prod_{l=1}^k \frac{1}{\text{deg}_f(v^l_p)} \sum_{\Phi} Z_q \prod_{l=1}^k w^{(l)}_q,
\]
where \( \Phi \) is the set of paths \( q \) from the input neurons to the output neuron \( (i,j) \) in the computation graph of \( \frac{\partial h^{(k)}_x}{\partial h^{(0)}_y}_p \). For each layer \( l \), \( w^{(l)}_q \) is the entry of \( W_l \) used in the \( q \)-th path. Finally, \( Z_q \in \{0,1\} \) indicates whether the \( q \)-th path is active (\( Z_q = 1 \)) or not (\( Z_q = 0 \)) due to the ReLU activation of entries in \( f^{(l)}_{v^l_p} \) along the \( q \)-th path. ReLU is not differentiable at 0; for simplicity, we assume the (sub)gradient at 0 is 0.

\medskip
\noindent\textbf{Effect of the D1 bypass.}
Design D1 adds a bypass that routes the PLM embedding $\mathbf{h}^{(0)}_x$ directly to the final representation, $\tilde{\mathbf{h}}^{(k)}_x = \mathbf{h}^{(k)}_x + B\,\mathbf{h}^{(0)}_x$, with a trainable bypass matrix $B$. Since the bypass is linear, differentiating with respect to $\mathbf{h}^{(0)}_y$ requires no product rule:
\[
\frac{\partial \tilde{\mathbf{h}}^{(k)}_x}{\partial \mathbf{h}^{(0)}_y} = \frac{\partial \mathbf{h}^{(k)}_x}{\partial \mathbf{h}^{(0)}_y} + B\,\mathbb{1}[x=y].
\]
The bypass leaves every off-diagonal block ($y\neq x$) unchanged and adds the block $B$ to the self term ($y=x$). Writing the influence score as the entrywise $\ell_1$ norm of the Jacobian block, and flattening $J:=\partial\mathbf{h}^{(k)}_x/\partial\mathbf{h}^{(0)}_x$ and $B$ to vectors, we obtain the following.

\begin{proposition}[Influence non-decrease under D1]
\label{prop:d1-influence}
For $y\neq x$, $I_{\mathrm{with}\,D_1}(x,y)=I_{\mathrm{without}\,D_1}(x,y)$. For the self term, if $J$ and $B$ are entrywise sign-compatible i.e. $J_i B_i\ge 0$ for every coordinate $i$, then
\[
I_{\mathrm{with}\,D_1}(x,x) = \|J+B\|_1 = \|J\|_1+\|B\|_1 = I_{\mathrm{without}\,D_1}(x,x) + \|B\|_1 .
\]
Hence $I_{\mathrm{with}\,D_1}(x,y)\ge I_{\mathrm{without}\,D_1}(x,y)$ for all $x,y$, with a strict increase of $\|B\|_1>0$ on the PLM-direct (self) path whenever $B\neq 0$.
\end{proposition}

\begin{proof}
The off-diagonal equality is immediate from the Jacobian identity above. For the self term, $|J_i+B_i|=|J_i|+|B_i|$ whenever $J_i B_i\ge0$; summing over coordinates gives $\|J+B\|_1=\|J\|_1+\|B\|_1$, and $\|B\|_1\ge0$, with equality iff $B=0$.
\end{proof}

\noindent\textbf{Scope of the claim.} The proposition concerns the \emph{unnormalized} influence. Under the normalization $I_x(y)=I(x,y)/\sum_z I(x,z)$, the bypass raises both the self numerator and the denominator by $\|B\|_1$, so the self (PLM-direct) normalized influence still increases, since $\tfrac{a+c}{b+c}\ge\tfrac{a}{b}$ for $0\le a\le b$ and $c\ge0$, while the off-diagonal normalized influence correspondingly decreases. Thus D1 \emph{concentrates} influence on the PLM-direct path, which is the intended effect.

\section{Proof of \Cref{prop:compact}}
\label{app:compact}

We justify design D2, replacing the concatenation $[\mathbf{H}_C;\mathbf{S}_T]$ by the sum $\mathbf{H}_C+\mathbf{S}_T$, by bounding the distortion this introduces in the inner-product geometry used for link prediction, and then showing that the quantity controlling this distortion, the overlap between the two signals, is small in practice.

\begin{proposition}[Lossless additive compression]
\label{prop:compact-lossless}
For a node pair $a$, write $\mathbf{h}^a:=[\mathbf{H}_C]_a$ and $\mathbf{s}^a:=[\mathbf{S}_T]_a$ in $\mathbb{R}^d$, with concatenated representation $z^a_{\mathrm{cat}}=(\mathbf{h}^a,\mathbf{s}^a)$ and additive representation $z^a_{\mathrm{sum}}=\mathbf{h}^a+\mathbf{s}^a$, so that $\langle z^a_{\mathrm{cat}},z^b_{\mathrm{cat}}\rangle=\langle \mathbf{h}^a,\mathbf{h}^b\rangle+\langle \mathbf{s}^a,\mathbf{s}^b\rangle$. Then for all node pairs $a,b$,
\[
\big|\langle z^a_{\mathrm{sum}},z^b_{\mathrm{sum}}\rangle-\langle z^a_{\mathrm{cat}},z^b_{\mathrm{cat}}\rangle\big|
\;\le\;\big|\langle \mathbf{h}^a,\mathbf{s}^b\rangle\big|+\big|\langle \mathbf{s}^a,\mathbf{h}^b\rangle\big|.
\]
In particular, if $\mathbf{H}_C$ and $\mathbf{S}_T$ have disjoint coordinate support the cross terms vanish and $\langle z^a_{\mathrm{sum}},z^b_{\mathrm{sum}}\rangle=\langle z^a_{\mathrm{cat}},z^b_{\mathrm{cat}}\rangle$ for all $a,b$: the additive and concatenated representations induce identical Gram matrices, so summation is lossless for any inner-product- or distance-based link predictor.
\end{proposition}

\begin{proof}
By bilinearity of the inner product,
\[
\langle \mathbf{h}^a+\mathbf{s}^a,\;\mathbf{h}^b+\mathbf{s}^b\rangle
=\big(\langle \mathbf{h}^a,\mathbf{h}^b\rangle+\langle \mathbf{s}^a,\mathbf{s}^b\rangle\big)
+\langle \mathbf{h}^a,\mathbf{s}^b\rangle+\langle \mathbf{s}^a,\mathbf{h}^b\rangle .
\]
The parenthesised term equals $\langle z^a_{\mathrm{cat}},z^b_{\mathrm{cat}}\rangle$; subtracting it and applying the triangle inequality gives the stated bound. If the supports are disjoint then $\langle \mathbf{h}^a,\mathbf{s}^b\rangle=\sum_k h^a_k s^b_k=0$, since every coordinate annihilates at least one factor, and likewise $\langle \mathbf{s}^a,\mathbf{h}^b\rangle=0$; hence the two Gram entries coincide.
\end{proof}

\noindent\textbf{Small overlap (empirical).} \Cref{prop:compact-lossless} reduces the validity of D2 to a single quantity, the overlap (cross inner products) of $\mathbf{H}_C$ and $\mathbf{S}_T$,which is small in practice. \Cref{fig:d2_clarification} (left) shows the overlapping rate between the non-zero supports of $\mathbf{H}_C$ and $\mathbf{S}_T$ decays toward zero as $|\mathcal{V}_s|$ grows, and (right) the AUC gap between the additive and concatenated variants vanishes accordingly. This complementarity is consistent with the latent-space view of \citet{Sarkar2011theoretical,mao2023revisiting}, in which structural and semantic proximity act as complementary link-formation mechanisms.

\section{Proof of \Cref{thm:d3}}
\label{app:d3}

Let $\mathcal{G}=(\mathcal{V},\mathcal{E})$ be a fixed graph and $G_\phi$ a $K$-layer GCN whose per-layer propagation operator has norm at most $\rho<\infty$. We make two Lipschitz assumptions: \textbf{(A1)} the PLM encoder $E_\theta$ is $L_E$-Lipschitz \emph{in its parameters} $\theta$, producing node features $\mathbf{X}^{(t)}=E_{\theta_t}(\text{text})$; and \textbf{(A2)} the feature-to-gradient map $\mathbf{X}\mapsto\nabla_\phi\,\ell\big(G_\phi(\mathbf{X});(i,j)\big)$ is $L_J$-Lipschitz, uniformly over training pairs $(i,j)$. Assumption~(A2) is a standard smoothness hypothesis; it is implied, for instance, by an $L_\ell$-Lipschitz loss with $L_g$-Lipschitz gradient composed with the $K$-layer GCN, which yields $L_J\le L_g\,L_\ell\,\rho^{K}$ , but we do not derive this constant and simply assume (A2). In D3 a node $v$ is re-encoded only when sampled as an endpoint; let $r_v(t)\le t$ be the most recent such iteration (its \emph{last-refresh time}) and $a_v(t):=t-r_v(t)$ its \emph{staleness age}. The cache thus holds $\widehat{\mathbf{X}}^{(t)}_v=E_{\theta_{r_v(t)}}(\text{text}_v)$, which differs from the fresh feature $\mathbf{X}^{(t)}_v=E_{\theta_t}(\text{text}_v)$ by the parameter drift accumulated since $r_v(t)$:
\[
\|\mathbf{X}^{(t)}_v-\widehat{\mathbf{X}}^{(t)}_v\|=\|E_{\theta_t}(\text{text}_v)-E_{\theta_{r_v(t)}}(\text{text}_v)\|\le L_E\,\|\theta_t-\theta_{r_v(t)}\|.
\]
We bound the gradient with respect to the GCN/predictor parameters $\phi$; the freshly re-encoded endpoints are exact, so error enters only through the stale rows. (Caching non-endpoint features also perturbs the PLM gradient $\nabla_\theta$; controlling that approximation is a deliberate design choice of D3 and lies outside the scope of the present bound.)

\begin{theorem}[Cache-staleness bound for D3]
\label{thm:d3-staleness}
Let $g^{\mathrm{full}}_t:=\nabla_\phi\,\ell\big(G_\phi(\mathbf{X}^{(t)});(i_t,j_t)\big)$ and $g^{\mathrm{cache}}_t:=\nabla_\phi\,\ell\big(G_\phi(\widehat{\mathbf{X}}^{(t)});(i_t,j_t)\big)$. Then
\[
\|g^{\mathrm{full}}_t-g^{\mathrm{cache}}_t\|\;\le\; C\sum_{v\in\mathcal{V}}\|\theta_t-\theta_{r_v(t)}\|,\qquad C:=L_J\,L_E .
\]
Consequently: \emph{(i)}~if every PLM update satisfies $\|\theta_{s+1}-\theta_s\|\le\eta$, then $\|\theta_t-\theta_{r_v(t)}\|\le a_v(t)\,\eta$, so $\|g^{\mathrm{full}}_t-g^{\mathrm{cache}}_t\|\le C\eta\sum_{v} a_v(t)$; with a periodic full refresh every $T$ iterations ($a_v(t)\le T$) this is at most $C\eta\,|\mathcal{V}|\,T$, i.e.\ uniformly bounded. \emph{(ii)}~if the updates are summable ($\sum_s\|\theta_{s+1}-\theta_s\|<\infty$, as under a convergent training schedule) \emph{and} every node is refreshed infinitely often ($r_v(t)\to\infty$ as $t\to\infty$,which holds, e.g., almost surely whenever each node has sampling probability bounded below by some $\pi_{\min}>0$), then $\|\theta_t-\theta_{r_v(t)}\|\to0$ for every $v$ and the cache error vanishes as $t\to\infty$.
\end{theorem}

\begin{proof}
By assumption~(A2), the $L_J$-Lipschitzness of the feature-to-gradient map,
\[
\|g^{\mathrm{full}}_t-g^{\mathrm{cache}}_t\|\le L_J\,\|\mathbf{X}^{(t)}-\widehat{\mathbf{X}}^{(t)}\|\le L_J\sum_{v\in\mathcal{V}}\|\mathbf{X}^{(t)}_v-\widehat{\mathbf{X}}^{(t)}_v\|,
\]
the last step bounding the matrix norm (Frobenius, or any $\ell_{2,1}$-type norm) by the sum of its row norms. Applying the per-node encoder-drift bound $\|\mathbf{X}^{(t)}_v-\widehat{\mathbf{X}}^{(t)}_v\|\le L_E\|\theta_t-\theta_{r_v(t)}\|$ from~(A1) gives the stated inequality with $C=L_J L_E$. For~(i), the telescoping identity $\theta_t-\theta_{r_v}=\sum_{s=r_v}^{t-1}(\theta_{s+1}-\theta_s)$ with the triangle inequality gives $\|\theta_t-\theta_{r_v}\|\le\sum_{s=r_v}^{t-1}\|\theta_{s+1}-\theta_s\|\le a_v(t)\,\eta$; summing over $v$ and using $a_v(t)\le T$ yields the uniform bound. For~(ii), $\|\theta_t-\theta_{r_v(t)}\|\le\sum_{s=r_v(t)}^{t-1}\|\theta_{s+1}-\theta_s\|\le\sum_{s\ge r_v(t)}\|\theta_{s+1}-\theta_s\|$, the tail of a convergent series; since $r_v(t)\to\infty$ its starting index diverges, so the tail,and hence the drift, tends to $0$.
\end{proof}

\noindent\textbf{Relation to coverage.} This bound does not vanish merely because every node has been seen once: it is governed by the staleness age $a_v(t)$ and the step size. It is controlled precisely when training stabilizes (shrinking updates) or the cache is periodically refreshed, matching practice, where late-training updates are small and refreshes are inexpensive.

\section{Complexity Analysis: Details}
\label{sec:app-complexity} 
\textbf{LLM Complexity Analysis}
Transformer-based LLM's time complexity is $O(l(|\mathcal{T}|d^2 + |\mathcal{T}|^2d))$, where $|\mathcal{T}|$, $h$, $l$ denotes token length per node, token hidden dimension, number of transformer layers.

\noindent \textbf{GNN Complexity Analysis:}
Let $\mathbf{A} \in \mathbb{R}^{N \times N}$ be the sparse adjacency matrix of a graph $\mathcal{G} = (\mathcal{V}, \mathcal{E})$, where $N = |\mathcal{V}|$ is the number of nodes and $E = |\mathcal{E}|$ is the number of edges. Let $\mathbf{X} \in \mathbb{R}^{N \times F}$ be a dense feature matrix with d-dimensional node embeddings.  The total computational complexity is given by  $O(d |\mathcal{E}|)$.
Thus overall complexity is $O(d^2 |\mathcal{T}| + d|\mathcal{T}|^2 + d|\mathcal{E}|)$

\section{Hardware Specification}
We run experiments on our benchmarks with a Horeka Cluster, which features a 20-core CPU, 140 GB Memory and a Nvidia A100 GPU with 40 GB GPU Memory. 

\section{Experiment Setting \& Hyper parameter Range}
\label{subsec:app-hyperparameter-gcn2lp}
\label{app:hyper_tune}

\subsection{LLM Fine-Tune details}
\label{subsec:app-lm-details}
We leverage the [EOS] token in LLaMA3 and the [CLS] token in sentence embedding models as node features. We have frozen the layer up to 6-th encoder layer in MiniLM \citep{Reimers2019SentenceBERTSE}; up to 32-th layer in Llama3 \citep{Dubey2024TheL3};  up to 12-th encoder layer in BERT and up to 23-th encoder layer in e5-large \citep{Wang2022TextEB}. We utilize full parameter fine-tuning in the proposed design. In summary, the proposed design ensures that the method is both memory-efficient, enabling training on a single A100 GPU with 40GB VRAM across all proposed datasets.

\subsection{GCN Fine-Tune Details}
We present the hyperparameter searching range in \Cref{tab:compare-gcn}, \Cref{tab:compare-plm} and \Cref{tab:llm_result}.

For the smaller graphs, Cora, Pubmed and Arxiv\_2023, we utilize the following hyper-parameter search space. However, it’s not feasible to tune over such large space for larger datasets. We utilize the average value of each optimized value from Cora, Pubmed, Arxiv\_2023. 

\noindent \textbf{GCN \citep{Kipf2016SemiSupervisedCW}, GAT \citep{Velickovic2017GraphAN}, and GraphSAGE \citep{Hamilton2017InductiveRL}:} The search space for these models includes learning rate: $\{10^{-2}, 10^{-3}, 10^{-4}\}$, batch size: $\{2^7, 2^8, 2^9, 2^{10}\}$, hidden layer dropout: $\{0.1, 0.3\}$, weight decay: $\{10^{-4}, 10^{-6}\}$, number of GCN layers: $\{1, 2, 3\}$, number of MLP layers: $\{1, 2, 3\}$, hidden dimension: $\{2^6, 2^7, 2^8\}$, and GCN output dimension: $\{2^6, 2^7, 2^8\}$. Following \citet{chamberlain_graph_2023}, we utilize a full adjacency matrix and remove the target link during training.

\noindent \textbf{NCNC \citep{Wang2023NeuralCN}:} The hyperparameters for NCNC are configured as follows: feature dropout: $\{0.0, 0.3, 0.7\}$, GNN learning rate: $\{10^{-3}, 10^{-4}\}$, predictor learning rate: $\{10^{-3}, 10^{-4}\}$, hidden dimension: $\{64, 256\}$, GNN dropout: $\{0.0, 0.2, 0.5\}$, and batch size: $\{2^7, 2^8, 2^9, 2^{10}\}$. We adopted the experimental settings from \citet{Wang2023NeuralCN}, including the removal of validation edges during training, jumping knowledge (JK) connections, and default settings for all remaining parameters.

\noindent \textbf{NeoGNN \citep{yun2021neognns}}: hidden channels: [64, 128, 256, 512, 1024, 2048, 4096, 8192], num of layers: [1,  2, 3, 4], number of mlp layers: [1, 2, 3, 4],  dropout': [0, 0.1, 0.2, 0.3, 0.4, 0.5], batch size: [128, 256, 512, 1024], learning rate: [0.01, 0.001, 0.0001]. We transferred the experiment setting from the original paper. 

\noindent \textbf{SEAL \citep{Zhang2020LabelingTA}}: hidden channels: [32, 64, 128, 256], batch size: [32, 64, 128, 256], lr: [0.001, 0.0001].  We transferred the experiment setting from the original paper, including text, structure embedding and 

\noindent \textbf{BUDDY} \citep{chamberlain_graph_2023}: hidden channels: [128, 256, 512], batch size: [512, 1024], learning rate: [0.01, 0.001, 0.0001], maximum of hash hops: [1, 2, 3], dropout for label: [0.1, 0.5], dropout for feature: [0.1, 0.3, 0.5], dropout of sign: [0.1, 0.3, 0.5, 0.7]. We transferred the experiment setting from the original paper, including text, structure embedding, and utilizing valid edge during training.  

\noindent \textbf{HLGNN \citep{zhang2024heuristiclearninggraphneural}:} {\sloppy The hyperparameter search space includes hidden channels: $\{128, 256, \dots, 8192\}$, batch size: $\{128, 256, 512, 1024\}$, and learning rate: $\{10^{-2}, 10^{-3}, 10^{-4}\}$. The initial parameter is selected from $\{0.1, 0.2, 0.3, 0.4, 0.5\}$, $\beta$ initialization from $\{RWR, KI\}$, number of MLP layers from $\{2, 3, 4\}$, and dropout from $\{0.1, \dots, 0.6\}$. Most experimental settings are adopted from the original work, with the exception of the data split strategy. While the original paper splits edges based on local structure using a 5\%/10\%/85\% (test/valid/train) ratio, we employ a random split with a 5\%/15\%/80\% ratio.\par}

\subsection{LLM Fine-Tune Details}
\noindent \textbf{MiniLM \citep{Reimers2019SentenceBERTSE}:} We utilize the [EOS] token for Llama-3 and the [CLS] token for sentence embedding models to extract node features. We froze the first six encoder layers of MiniLM-L6-v2 \citep{Reimers2019SentenceBERTSE} and set the feature dimension (\textit{feat\_shrink}) to 768. Training hyperparameters include an attention dropout of 0.1, a batch size of 128, and a learning rate of $10^{-4}$. The model is trained for 250 epochs with a warmup period of 0.6 epochs and a weight decay of 0.0. Evaluation patience is set to 1, and gradient accumulation is fixed at 1 step.

\noindent \textbf{e5-large \citep{Wang2022TextEB}:} We froze the first 23 encoder layers of e5-large and fine-tuned the remaining layers. The feature dimension (\textit{feat\_shrink}) is configured to 768. Training parameters include an attention dropout of 0.1, a batch size of 128, a classification dropout of 0.4, and a general dropout of 0.3. The model is trained for 250 epochs with a learning rate of $10^{-4}$ and a warmup period of 0.6 epochs. Weight decay is set to 0.0, and evaluation patience is set to a sufficiently large value to effectively disable early stopping.

\noindent \textbf{MPNet \citep{Song2020MPNetMA}}: Same as Minilm, we use CLS token embedding, we frozen up to 32-th layer in Mpnet and fine-tune the rest layers. Feat shrink, is 768. The model name is `sentence-transformers/all-mpnet-base-v2` from \citet{Reimers2019SentenceBERTSE}. For training, the parameters include an attention dropout rate of 0.1, a batch size of 256, a classification dropout rate of 0.4, and a general dropout rate of 0.3. The model is trained for 250 epochs, with evaluation patience set to a very large number, effectively disabling early stopping. Gradient accumulation steps are set to 4, and the learning rate is 0.0001. The warmup period is set to 0.6 epochs, and weight decay is set to 0.0.

\noindent \textbf{ZO-MLP}: We utilize BERT, MiniLM, e5-large and Llama 3 8B as inference models to only embed text without training. 

\noindent \textbf{FT-MLP}: We use the above parameters to find these sentence transformers followed by an MLP. 

\noindent \textbf{ZO-GCN}: This setting requires two separate steps. First, we embed the raw text with the above sentence transformers with given parameters and save the embedded text as original node features. In the second step, we load this embedded text as node features to aggregated features using different GCNs with the same setting as above. 

\noindent \textbf{FT-GCN}: We leverage the same fine-tuned setting, the same parameters of sentence transformers and GCNs as above. We use the same split and full adjacency matrix during the training as category GCNs. 

\noindent \textbf{\method}: We utilize the same fine-tuned setting of Mpnet, Minilm, and e5-large in our proposed method without any additional parameter tuning, We leverage the same data split and full adjacency matrix for common neighbor and GCN, validation edge during training. In addition,  the same parameters for optimizer and GCNs from \citep{Kipf2016SemiSupervisedCW}. The parameters for GCN are listed below: an input channel size of 1433, a hidden channel size of 256, and an output channel size of 32. It consists of 1 multi-layer perceptron (MLP) layer and 3 such graph convolution layers. We only tune the parameters for the optimizer including learning rate, batch size and weight decay. The common-neighbor threshold is fixed to $\tau=2$ and the mixing weight of the thresholded common neighbor to $\beta=0.1$; both are held constant across all datasets rather than tuned per-dataset.

\newpage
\section{Benchmark Result of all Models}

\label{sec:app-full-result}
\begin{table}[htbp]
  \centering
  \caption{Benchmark result of local and global heuristic, graph embedding, vanilla GCNs, GCN2LPs, LLM-Inf, LLM-FT. }
  \label{tab:app-tab00}
  \resizebox{0.95\linewidth}{!}{
  \begin{tabular}{lccccccc}
    \toprule
    \multicolumn{1}{c}{Models} & \multicolumn{2}{c}{Cora} & \multicolumn{2}{c}{PubMed} & \multicolumn{2}{c}{Arxiv 2023} \\
    \multicolumn{1}{c}{} & Hits@100 & AUC & Hits@100 & AUC & Hits@100 & AUC \\
    \midrule     
    RotatE & 56.60 $\pm$ 6.29 & 61.57 $\pm$ 3.88 & 38.58 $\pm$ 2.59 & 80.89 $\pm$ 0.40 & 34.57 $\pm$ 0.84 & 76.73 $\pm$ 0.44 \\
    TransE & 57.00 $\pm$ 7.06 & 64.19 $\pm$ 3.83 & 50.04 $\pm$ 2.63 & 85.29 $\pm$ 0.97 & 30.22 $\pm$ 1.96 & 73.03 $\pm$ 1.04 \\
    ComplEx & 73.59 $\pm$ 0.90 & 75.71 $\pm$ 1.09 & 53.87 $\pm$ 3.09 & 89.60 $\pm$ 0.51 & 32.17 $\pm$ 2.38 & 73.31 $\pm$ 0.61\\
    DistMult & 73.12 $\pm$ 0.28 & 76.53 $\pm$ 0.80 & 54.82 $\pm$ 2.95 & 89.91 $\pm$ 0.53 & 32.37 $\pm$ 1.62 & 73.25 $\pm$ 0.39\\
    BERT-Inf & 56.90 $\pm$ 3.26 & 65.09 $\pm$ 1.41 & 48.73 $\pm$ 1.43 & 88.96 $\pm$ 0.31 & 48.74 $\pm$ 1.15 & 86.37 $\pm$ 0.27 \\
    e5-large-v2 & 83.10 $\pm$ 0.80 & 83.87 $\pm$ 0.23 & 82.59 $\pm$ 0.26 & 96.73 $\pm$ 0.03 & 84.09 $\pm$ 0.24 & 95.72 $\pm$ 0.01 \\
    MiniLM-L6-v2 & 92.99 $\pm$ 0.00 & 91.22 $\pm$ 0.04 & 81.90 $\pm$ 0.03 & 96.20 $\pm$ 0.00 & 77.62 $\pm$ 0.03 & 95.22 $\pm$ 0.00 \\
    Llama-3-8B & 95.64 $\pm$ 0.41 & 92.60 $\pm$ 0.12 & 89.01 $\pm$ 0.53 & 98.09 $\pm$ 0.10 & 89.91 $\pm$ 0.19 & 97.62 $\pm$ 0.04 \\
    \midrule
    Llama-BUDDY & 97.47 $\pm$ 1.44 & 93.97 $\pm$ 0.88 & 94.60 $\pm$ 0.98 & 98.75 $\pm$ 0.18 & 84.93 $\pm$ 1.68 & 96.67 $\pm$ 0.38 \\
    e5-large-BUDDY & 97.86 $\pm$ 0.66 & 94.03 $\pm$ 0.77 & 93.95 $\pm$ 0.34 & 98.72 $\pm$ 0.06 & 77.43 $\pm$ 3.01 & 95.27 $\pm$ 0.31 \\
    BERT-BUDDY & 81.42 $\pm$ 2.86 & 79.62 $\pm$ 2.44 & 78.54 $\pm$ 1.26 & 96.25 $\pm$ 0.17 & 43.30 $\pm$ 1.65 & 77.93 $\pm$ 0.62 \\
    MiniLM-BUDDY & 97.55 $\pm$ 1.32 & 94.74 $\pm$ 0.99 & 93.56 $\pm$ 0.66 & 98.68 $\pm$ 0.07 & 47.12 $\pm$ 2.93 & 77.59 $\pm$ 0.35 \\
    Llama-NCN & 98.26 $\pm$ 0.77 & 96.23 $\pm$ 0.28 & 96.30 $\pm$ 0.62 & \cellcolor{gray!5}99.11 $\pm$ 0.07 & \cellcolor{gray!10}92.64 $\pm$ 0.52 & \cellcolor{gray!10}98.24 $\pm$ 0.17 \\
    e5-large-NCN & 98.81 $\pm$ 0.62 & 96.72 $\pm$ 0.67 & 96.33 $\pm$ 0.70 & 99.10 $\pm$ 0.09 & 90.79 $\pm$ 1.38 & 97.82 $\pm$ 0.22 \\
    BERT-NCN & 82.93 $\pm$ 3.65 & 84.91 $\pm$ 2.53 & 81.09 $\pm$ 2.46 & 97.24 $\pm$ 0.20 & 66.20 $\pm$ 0.62 & 86.94 $\pm$ 0.27 \\
    MiniLM-NCN & 98.58 $\pm$ 0.60 & 96.93 $\pm$ 0.54 & 95.49 $\pm$ 0.66 & 98.97 $\pm$ 0.10 & 66.47 $\pm$ 0.84 & 86.69 $\pm$ 0.28 \\
    Llama-NCNC & 98.73 $\pm$ 0.65 & 95.63 $\pm$ 0.77 & 92.39 $\pm$ 1.46 & 98.54 $\pm$ 0.18 & \cellcolor{gray!5}91.90 $\pm$ 1.33 & 97.97 $\pm$ 0.36 \\
    e5-large-NCNC & 98.81 $\pm$ 0.74 & 96.31 $\pm$ 0.68 & \cellcolor{gray!5}96.69 $\pm$ 0.56 & 99.14 $\pm$ 0.10 & 90.46 $\pm$ 1.14 & 97.72 $\pm$ 0.16 \\
    BERT-NCNC & 84.11 $\pm$ 2.70 & 84.05 $\pm$ 2.85 & 82.48 $\pm$ 1.43 & 97.40 $\pm$ 0.20 & 68.50 $\pm$ 3.49 & 89.20 $\pm$ 0.30 \\
    MiniLM-NCNC & 98.81 $\pm$ 0.49 & 96.82 $\pm$ 0.62 & 96.11 $\pm$ 0.60 & 99.07 $\pm$ 0.07 & 70.61 $\pm$ 2.24 & 88.42 $\pm$ 0.72 \\
    Llama-NeoGNN & 83.00 $\pm$ 8.24 & 87.08 $\pm$ 3.07 & 69.42 $\pm$ 3.09 & 94.90 $\pm$ 0.52 & 63.67 $\pm$ 8.29 & 86.94 $\pm$ 1.22 \\
    e5-large-NeoGNN & 82.06 $\pm$ 4.54 & 87.06 $\pm$ 1.69 & 71.13 $\pm$ 2.19 & 92.21 $\pm$ 5.85 & 69.09 $\pm$ 9.97 & 88.32 $\pm$ 1.81 \\
    BERT-NeoGNN & 87.59 $\pm$ 3.79 & 87.77 $\pm$ 2.30 & 70.85 $\pm$ 1.71 & 95.08 $\pm$ 0.38 & 65.43 $\pm$ 8.49 & 86.22 $\pm$ 1.93 \\
    MiniLM-NeoGNN & 87.59 $\pm$ 6.28 & 88.66 $\pm$ 2.63 & 70.62 $\pm$ 1.71 & 95.08 $\pm$ 0.38 & 68.60 $\pm$ 11.86 & 86.10 $\pm$ 0.75 \\
    \midrule
    BERT-FT & 96.99 $\pm$ 1.36  & 92.88 $\pm$ 0.99  & 84.45 $\pm$ 2.92  & 97.32 $\pm$ 0.44  & 87.56 $\pm$ 2.05 & 96.98 $\pm$ 0.31  \\
    e5-large-v2-FT & 96.92 $\pm$ 1.35 & 94.27 $\pm$ 0.85  & 87.23 $\pm$ 1.60  & 97.79 $\pm$ 0.14  & 89.35 $\pm$ 1.33 & 97.39 $\pm$ 0.33  \\
    MiniLM-L6-v2-FT & 96.68 $\pm$ 1.69  & 93.98 $\pm$ 0.85  & 86.80 $\pm$ 1.98  & 97.64 $\pm$ 0.36  & 88.38 $\pm$ 1.06 & 97.25 $\pm$ 0.36  \\
    all-mpnet-base-v2  & 97.78 $\pm$ 0.66 & 94.71 $\pm$ 1.16 & 90.69 $\pm$ 2.49 & 98.06 $\pm$ 0.19 & 91.44 $\pm$ 0.75 & 97.36 $\pm$ 0.33 \\
     \midrule
    MiniLM-GAT-CT  & 76.99 $\pm$ 6.58  & 76.09 $\pm$ 4.18  & 43.75 $\pm$ 5.46  & 90.11 $\pm$ 1.81  & 25.18 $\pm$ 4.17 & 80.31 $\pm$ 3.08 \\
    MiniLM-SAGE-CT & 82.01 $\pm$ 4.19  & 80.15 $\pm$ 2.74  & 63.18 $\pm$ 1.34  & 94.10 $\pm$ 0.21  & 66.06 $\pm$ 2.60 & 92.63 $\pm$ 0.51 \\
    mpnet-GIN-CT & 97.55 $\pm$ 1.86  & 92.30 $\pm$ 2.28  & 63.42 $\pm$ 2.77  & 94.82 $\pm$ 0.42  & 64.38 $\pm$ 3.22 & 91.97 $\pm$ 1.07 \\
    mpnet-GAT-CT  & 86.96 $\pm$ 6.71  & 86.64 $\pm$ 5.61  & 52.43 $\pm$ 4.18  & 91.75 $\pm$ 0.86  & 29.43 $\pm$ 3.60 & 82.50 $\pm$ 3.28 \\
    mpnet-SAGE-CT & 93.88 $\pm$ 0.28  & 89.64 $\pm$ 0.96  & 71.62 $\pm$ 2.78  & 94.60 $\pm$ 0.58  & 66.28 $\pm$ 4.50 & 92.63 $\pm$ 0.87 \\
     \midrule
    MiniLM-\method-C & \cellcolor{gray!30}99.92 $\pm$ 0.18  & \cellcolor{gray!20}98.24 $\pm$ 0.95  & \cellcolor{gray!20}99.94 $\pm$ 0.08  & \cellcolor{gray!40}99.73 $\pm$ 0.03  & \cellcolor{gray!40}98.16 $\pm$ 1.73 & \cellcolor{gray!40}99.05 $\pm$ 0.25 \\
    MiniLM-\method & \cellcolor{gray!10}99.84 $\pm$ 0.35  & \cellcolor{gray!10}97.79 $\pm$ 0.66  & 89.22 $\pm$ 4.96  & 98.30 $\pm$ 0.51  & 91.44 $\pm$ 1.34 & \cellcolor{gray!5}98.22 $\pm$ 0.30 \\
    mpnet-\method-C  & 96.28 $\pm$ 8.31  & 94.68 $\pm$ 6.81  & \cellcolor{gray!30}99.95 $\pm$ 0.08  & \cellcolor{gray!20}99.42 $\pm$ 0.38  & 77.99 $\pm$ 17.20 & 92.88 $\pm$ 6.53 \\
    mpnet-\method  & \cellcolor{gray!40}100.00 $\pm$ 0.00  & \cellcolor{gray!40}98.78 $\pm$ 1.02  & \cellcolor{gray!10}97.13 $\pm$ 1.74  & \cellcolor{gray!10}99.34 $\pm$ 0.22  & \cellcolor{gray!20}94.85 $\pm$ 3.15 & \cellcolor{gray!30}98.79 $\pm$ 0.49 \\
    e5-large-\method-C  & \cellcolor{gray!30}99.92 $\pm$ 0.18  & \cellcolor{gray!30}98.36 $\pm$ 1.15  & \cellcolor{gray!40}100.00 $\pm$ 0.00  & \cellcolor{gray!30}99.49 $\pm$ 0.18  & \cellcolor{gray!30}98.01 $\pm$ 0.72 & \cellcolor{gray!20}98.51 $\pm$ 0.12 \\
    e5-large-\method & \cellcolor{gray!5}98.89 $\pm$ 1.02  & \cellcolor{gray!5}97.13 $\pm$ 1.22  & 88.41 $\pm$ 1.14  & 98.18 $\pm$ 0.12  & 87.71 $\pm$ 1.48 & 97.51 $\pm$ 0.29 \\
    \bottomrule
  \end{tabular}
  }
\end{table}

\begin{table}[htbp]
  \centering
  \caption{Benchmark result of local and global heuristic, graph embedding, vanilla GCNs, GCN2LPs, ZO, FT. }
  \label{tab:app-tab01}
  \resizebox{0.95\linewidth}{!}{
  \begin{tabular}{lccccccc}
    \toprule
    \multicolumn{1}{c}{Models} & \multicolumn{2}{c}{Cora} & \multicolumn{2}{c}{PubMed} & \multicolumn{2}{c}{Arxiv 2023} \\
    \multicolumn{1}{c}{} & Hits@50 & MRR & Hits@50 & MRR & Hits@50 & MRR \\
    \midrule     
    MiniLM-GAT-CT & 54.23 $\pm$ 4.08  & 13.74 $\pm$ 5.21  & 29.44 $\pm$ 2.84  & 4.26 $\pm$ 1.75  & 13.76 $\pm$ 2.22 & 2.62 $\pm$ 0.55 \\
    MiniLM-SAGE-CT  & 63.04 $\pm$ 9.23  & 15.09 $\pm$ 1.35  & 45.31 $\pm$ 3.83  & 6.70 $\pm$ 3.61  & 49.11 $\pm$ 4.22 & 11.48 $\pm$ 3.41 \\
    mpnet-GIN-CT & 89.01 $\pm$ 5.54  & 29.06 $\pm$ 7.96  & 46.82 $\pm$ 4.22  & 11.86 $\pm$ 3.68  & 55.11 $\pm$ 3.70 & 18.88 $\pm$ 5.89 \\
    mpnet-GAT-CT & 74.90 $\pm$ 9.22  & 23.16 $\pm$ 11.10  & 35.82 $\pm$ 3.81  & 5.36 $\pm$ 1.69  & 19.50 $\pm$ 1.91 & 4.49 $\pm$ 0.91 \\
    mpnet-SAGE-CT  & 82.01 $\pm$ 4.19  & 25.34 $\pm$ 8.06  & 57.58 $\pm$ 3.78  & 11.91 $\pm$ 3.58  & 52.97 $\pm$ 5.05 & 14.51 $\pm$ 3.37 \\
    \midrule
    RotatE & 27.51 $\pm$ 6.39 & 2.98 $\pm$ 0.41 & 26.10 $\pm$ 2.89 & 4.41 $\pm$ 1.24 & 24.76 $\pm$ 1.31 & 5.65 $\pm$ 1.20 \\
    TransE &  32.17 $\pm$ 6.52 & 3.97 $\pm$ 0.90 & 38.58 $\pm$ 3.64 & 7.74 $\pm$ 1.71 & 21.27 $\pm$ 1.81 & 5.11 $\pm$ 1.16\\
    ComplEx & 62.53 $\pm$ 1.71 & 17.82 $\pm$ 2.58 & 38.14 $\pm$ 3.55 & 6.50 $\pm$ 2.48 & 22.33 $\pm$ 1.05 & 5.16 $\pm$ 1.07\\
    DistMult & 64.90 $\pm$ 1.59 & 19.28 $\pm$ 2.11 & 39.09 $\pm$ 3.93 & 6.90 $\pm$ 2.50 & 21.32 $\pm$ 0.45 & 5.02 $\pm$ 1.24\\
    \midrule
    Bert & 35.79$\pm$2.50  & 3.42$\pm$0.47  & 36.12$\pm$0.37  & 6.56$\pm$0.70  & 37.66$\pm$1.57  & 10.04$\pm$0.85  \\
    MiniLM        & 83.39$\pm$0.00 & 34.29$\pm$4.10 & 66.35$\pm$0.29 & 21.54$\pm$0.11 & 68.15$\pm$0.09 & 16.91$\pm$0.18 \\
    e5-large-v2   & 64.35$\pm$1.56 & 24.40$\pm$2.48 & 71.32$\pm$0.86 & 21.79$\pm$1.58 & 75.03$\pm$0.28 & 21.69$\pm$0.03 \\
    Llama-3-8B    & 89.15$\pm$0.72  & 31.19$\pm$3.49 & 79.87$\pm$1.19 & 22.87$\pm$4.47
    & 83.18$\pm$1.19  & 22.85$\pm$1.12 \\
    \midrule
    Llama-BUDDY & 91.23 $\pm$ 1.66 & 29.10 $\pm$ 5.72 & 88.10 $\pm$ 2.08 & 24.57 $\pm$ 3.66 & 75.15 $\pm$ 3.55 & 28.74 $\pm$ 5.29 \\
    e5-large-BUDDY & 92.17 $\pm$ 1.32 & 34.19 $\pm$ 10.77 & 88.07 $\pm$ 0.57 & 22.90 $\pm$ 5.86 & 64.52 $\pm$ 3.00 & 20.79 $\pm$ 4.67 \\
    BERT-BUDDY & 63.32 $\pm$ 5.74 & 14.64 $\pm$ 3.70 & 62.03 $\pm$ 3.16 & 13.15 $\pm$ 2.84 & 35.19 $\pm$ 1.54 & 10.20 $\pm$ 2.88 \\
    MiniLM-BUDDY & 92.41 $\pm$ 1.62 & 41.05 $\pm$ 11.41 & 85.98 $\pm$ 1.19 & 25.45 $\pm$ 4.86 & 38.33 $\pm$ 2.89 & 11.31 $\pm$ 1.78 \\
    Llama-NCN & 94.94 $\pm$ 0.85 & 49.62 $\pm$ 10.03 & 90.05 $\pm$ 1.66 & 30.73 $\pm$ 10.11 & \cellcolor{gray!20}86.43 $\pm$ 0.24 & 32.48 $\pm$ 12.29 \\
    e5-large-NCN & 95.57 $\pm$ 1.02 & 49.65 $\pm$ 14.61 & \cellcolor{gray!5}90.97 $\pm$ 1.77 & 32.19 $\pm$ 4.80 & 84.36 $\pm$ 0.56 & 32.23 $\pm$ 14.04 \\
    BERT-NCN & 73.04 $\pm$ 4.08 & 43.37 $\pm$ 8.81 & 67.95 $\pm$ 3.42 & 23.16 $\pm$ 6.75 & 61.93 $\pm$ 0.78 & \cellcolor{gray!5}37.05 $\pm$ 8.86 \\
    MiniLM-NCN & \cellcolor{gray!10}98.58 $\pm$ 0.60 & 46.59 $\pm$ 11.01 & 88.85 $\pm$ 0.89 & 33.60 $\pm$ 7.51 & 61.95 $\pm$ 0.54 & 31.86 $\pm$ 13.16 \\
    Llama-NCNC & 95.57 $\pm$ 1.02 & 27.45 $\pm$ 7.86 & 84.65 $\pm$ 1.95 & 20.51 $\pm$ 9.80 & \cellcolor{gray!10}84.68 $\pm$ 1.72 & 27.16 $\pm$ 11.48 \\
    e5-large-NCNC & 96.13 $\pm$ 1.13 & 39.23 $\pm$ 12.99 & 90.86 $\pm$ 1.95 & 27.02 $\pm$ 5.96 & 83.24 $\pm$ 1.20 & 25.14 $\pm$ 9.39 \\
    BERT-NCNC & 77.47 $\pm$ 1.77 & 25.39 $\pm$ 12.42 & 72.80 $\pm$ 1.78 & 23.49 $\pm$ 3.07 & 58.83 $\pm$ 3.91 & 22.80 $\pm$ 2.55 \\
    MiniLM-NCNC & 96.13 $\pm$ 1.20 & 38.96 $\pm$ 13.20 & 90.32 $\pm$ 1.52 & 22.56 $\pm$ 3.30 & 65.65 $\pm$ 1.80 & 29.10 $\pm$ 3.83 \\
    Llama-NeoGNN & 77.63 $\pm$ 4.33 & 47.11 $\pm$ 8.51 & 64.00 $\pm$ 0.87 & \cellcolor{gray!5}33.82 $\pm$ 3.81 & 56.20 $\pm$ 10.94 & 22.39 $\pm$ 13.38 \\
    e5-large-NeoGNN & 78.91 $\pm$ 2.49 & 47.04 $\pm$ 8.63 & 64.74 $\pm$ 2.03 & 32.65 $\pm$ 5.49 & 63.56 $\pm$ 11.23 & 24.70 $\pm$ 11.08 \\
    BERT-NeoGNN & 80.08 $\pm$ 4.22 & 45.50 $\pm$ 8.99 & 64.10 $\pm$ 1.50 & 33.01 $\pm$ 5.49 & 57.40 $\pm$ 9.79 & 23.92 $\pm$ 8.63 \\
    MiniLM-NeoGNN & 81.35 $\pm$ 5.27 & 43.83 $\pm$ 6.82 & 65.13 $\pm$ 0.26 & \cellcolor{gray!10}33.89 $\pm$ 6.18 & 62.53 $\pm$ 14.04 & 21.49 $\pm$ 18.99 \\
    \midrule 
    BERT-FT & 89.17$\pm$2.86  & 30.90$\pm$4.33  & 73.70$\pm$4.01  & 17.11$\pm$3.90  & 77.75$\pm$3.46  & 29.54$\pm$3.98  \\
    e5-large-v2-FT & 92.09$\pm$1.70  & 38.63$\pm$9.39  & 76.26$\pm$2.55  & 19.75$\pm$5.81  & 80.48$\pm$2.52  & 31.73$\pm$6.62  \\
    MiniLM-FT & 92.49$\pm$2.13  & 35.55$\pm$5.82  & 75.87$\pm$3.72  & 20.79$\pm$6.32  & 80.20$\pm$2.62  & 29.86$\pm$5.82  \\
    Mpnet-FT  & 93.44 $\pm$ 1.64 & 22.55 $\pm$ 10.71 & 63.27 $\pm$ 31.76 & 9.38 $\pm$ 3.12 & 82.72 $\pm$ 1.28 & 8.42 $\pm$ 6.49 \\

     \midrule
    MiniLM-\method-C & \cellcolor{gray!30}99.92$\pm$0.18  & 41.52$\pm$19.50  & \cellcolor{gray!40}99.91 $\pm$ 0.09  & \cellcolor{gray!30}44.99 $\pm$ 10.82  & \cellcolor{gray!40}90.61 $\pm$ 2.25 & 35.47 $\pm$ 10.91 \\
    MiniLM-\method & \cellcolor{gray!5}98.34$\pm$0.59  & \cellcolor{gray!20}60.84$\pm$7.75  & 78.56 $\pm$ 6.32  & 22.72 $\pm$ 1.51  & \cellcolor{gray!05}84.57 $\pm$ 1.93 & \cellcolor{gray!20}40.27 $\pm$ 11.91 \\
    mpnet-\method-C  & 93.28$\pm$14.16  & 28.92$\pm$7.14  & \cellcolor{gray!30}99.27 $\pm$ 1.19  & 23.99 $\pm$ 11.63  & 73.09 $\pm$ 16.32 & 14.68 $\pm$ 6.17 \\
    mpnet-\method  & \cellcolor{gray!40}100.00$\pm$0.00  & \cellcolor{gray!40}68.43$\pm$14.23  & \cellcolor{gray!10}91.67 $\pm$ 4.96  & 31.66 $\pm$ 5.33  & \cellcolor{gray!30}89.17 $\pm$ 5.45 & \cellcolor{gray!40}45.70 $\pm$ 3.88 \\
    e5-large-\method-C  & \cellcolor{gray!30}99.92$\pm$0.18  & 41.46$\pm$25.49  & \cellcolor{gray!20}99.11 $\pm$ 1.54  & 21.66 $\pm$ 9.66  & 83.97 $\pm$ 4.23 & 12.66 $\pm$ 2.66 \\
    e5-large-\method & 96.29$\pm$2.08  & \cellcolor{gray!30}65.26$\pm$11.52  & 77.34 $\pm$ 2.19  & 23.80 $\pm$ 3.29  & 80.01 $\pm$ 2.53 & \cellcolor{gray!30}42.02 $\pm$ 5.56 \\
    \bottomrule
  \end{tabular}
  }
\end{table}

\newpage
\section{Pseudocode \& Pipeline}

\begin{algorithm*}[!ht]
\caption{method: Framework for Link Prediction in Text Attributed Graph}
\label{algo:method}
\SetNoFillComment
{\footnotesize 
\SetKwInput{Input}{Input}\SetKwInput{Output}{Output}
\SetKwInput{HyperParams}{Hyper-parameters}\SetKwInput{TrainParams}{Model Parameters}
\Input{( $\mathbf{A}, \mathbf{X}, \mathbf{T}, \mathcal{Y}, \mathcal{E}_{\text{train}}$)\\
\quad \quad \quad Graph Adjacency Matrix $\mathbf{A} \in \{0,1\}^{n\times n}$, Init Node Feature $\mathbf{X}$; \\
\quad \quad \quad Node Feature In Raw Text $\mathbf{T}$; \\
\quad \quad \quad Set of positive and negative edge Labels $\mathcal{Y} \in \{1, 0\}$; \\
\quad \quad \quad Labeled positive and negative Edges $\mathcal{E}_{\text{train}}^+, \mathcal{E}_{\text{train}}^- \in \mathcal{E}= \{(v_i , v_j) \subseteq (\mathcal{V} \times \mathcal{V})\}$; \\
}

\HyperParams{
Dropout Rate; 
Dimension of Feature Embedding $p$; %
}
\TrainParams{
PLM: $\mathbf{W}_{\mathcal{L}} \in \mathbb{R}^{F\times p}$; GCN: $\mathbf{W}_{\mathcal{G}} \in \mathbb{R}^{ p \times |\mathcal{Y}|}$ }
\Output{Edge class label vector $\mathbf{Y}_{i, j}$}
\Begin{
\tcc{All new variables defined below are initialized as zero}
\vspace{0.15cm}
\tcc{{Step 1: Text Embedding $\rightarrow$ Semantic Proximity (S2)}}
\For{$(i\;, j) \in \mathcal{E}_{\texttt{train}}^+ \cup \mathcal{E}_{\texttt{train}}^-$}{
    $\mathbf{x}^{T}_{i}  \leftarrow \sigma\left(\mathbf{W}_{\mathcal{L}}\mathbf{t}_{i}\right)$ 
    $ \mathbf{x}^{T}_{j} \leftarrow \sigma \left(\mathbf{W}_{\mathcal{L}} \mathbf{t}_{j}\right)$
    \tcc{Textual Embeddings are stored in the matrix $\mathbf{X}^T$}
\vspace{0.15cm}

\tcc{{Step 2: Neighborhood Aggregation $\rightarrow$ Propagated Proximity (S3)}}
\tcc{Calculate Embedding for $v_i$ and $v_j$ and their Structural Proximity}
\vspace{0.15cm}
    $\mathbf{x}_i \leftarrow \mathbf{x}_i^{T}$ \hfill
    \vspace{0.15cm}
    $\mathbf{x}_j \leftarrow \mathbf{x}_j^{T}$ \hfill \tcc{Replace  $\mathbf{x}_i, \mathbf{x}_j$ of current link with Textual Embedding $\mathbf{x}_i^T, \mathbf{x}_j^T$}
    \For{$k \leftarrow 1$ \KwTo $K$}{\tcc{symmetric degree-normalization of matrices $\mathbf{S} = \tilde{\mathbf{D}}^{-\frac{1}{2}} \tilde{\mathbf{A}} \tilde{\mathbf{D}}^{-\frac{1}{2}} $}
    \vspace{0.15cm}
    $\mathbf{H}^{k}\leftarrow \mathbf{S} \mathbf{H}^{k-1}\mathbf{W}_{\mathcal{G}}$ \hfill  \tcc{Final node embedding is $\mathbf{H}^K$}
  }
$\mathbf{H}^{K} = f \left( \tilde{\mathbf{A}}_{\text{sym}}^K\mathbf{X} \mathbf{W}_{\mathcal{G}}\right) $
\vspace{0.15cm}

\tcc{{Step 3: Structural Feature Embedding $\rightarrow$ Structural Proximity (S1)}}
    $\mathbf{A} \leftarrow \mathrm{dropout}(\mathbf{A})$ \hfill 
    
    \tcc{Calculate Common Neighbor for $v_i$ and $v_j$}
    \vspace{0.15cm}
    $\mathbf{H}_C \leftarrow  \mathbb{I}\left[\texttt{SUM}(\mathbf{A}_i \odot \mathbf{A}_j) > {\color{blue}\tau} \right] > 0 $
   
   \tcc{$\mathbb{I}$ is a element-wise indicator function for matrix}
    \vspace{0.15cm}
    \tcc{Pairwise Text Feature Proximity in Dot Product}
    $\mathbf{S}_T = \mathbf{W}(\mathbf{X}^T_i \odot \mathbf{X}^T_j)$
    
    \vspace{0.15cm}

}
\vspace{0.15cm}
\tcc{Structural and Textual Feature Concatenation}
$\mathbf{R}  \leftarrow \left(\mathbf{H}^K \; \Vert \; \mathbf{S}_T + \beta \mathbf{H}_C\right)
$ \hfill
\tcc{$\beta$: weight of Structural Proximity}

$\mathbf{R}  \leftarrow \left(\mathbf{R}_{-}, \mathbf{R}_{+} \right)$    \vspace{0.15cm}

\tcc{Step 4: Binary Classification}
$\mathbf{R} \leftarrow \mathrm{dropout}(\mathbf{R})$ \hfill 

\For{$(i\;, j) \in \mathcal{E}_{\texttt{test}}$}{
    $\mathbf{p}_{ij} \leftarrow \mathrm{softmax}(\mathbf{r}_{ij} \mathbf{W})$\;
    $\mathbf{Y}_{i, j} \leftarrow \arg \max(\mathbf{p}_{ij})$ \hfill \tcc{edge label}
    }
  }
}
\end{algorithm*}

\section{Proposed Dataset: Details}
\label{sec:app-dataset}
\begin{table}[ht]
    \centering
    \caption{Graph statistics of the benchmark datasets across
    different scales. We report dataset size and density, connectivity and hierarchy, scale-free characteristics.
      We additionally include the mean and minimum--maximum
      lengths of the textual descriptors associated with each node,
       together with the total number of textual data points.
}
    \begin{adjustbox}{width=\textwidth}
    \begin{tabular}{lccccccccc}
        \toprule
        & \textbf{Pwc\_small} & \textbf{Cora} & \textbf{PubMed} & \textbf{Arxiv\_2023} & \textbf{History} & \textbf{Photo} & \textbf{Pwc\_medium} & \textbf{Ogbn\_arxiv} & \textbf{Citationv8}\\
        \midrule
        \textbf{\#Nodes $|\mathcal{V}|$} & 140 & 2708 & 19717 & 46198 & 36655 & 47420 & 86795 & 169343 & 1106759\\
        \textbf{\#Edges $|\mathcal{E}|$} & 798 & 10858 & 88648 & 77726 & 496248 & 872602 & 933411 & 2315598 & 12227452 \\
        \textbf{File Size (Mb)} & 0.139 & 2.57 & 31.01 & 57.26 & 57.12 & 37.44 & 99.58 & 190.27 & 1034.18 \\
        \textbf{\#Split} & R & R & R & R & R & R & R & R & R\\
        \midrule
        \textbf{Avg Deg (G)} & 8.51 & 11.63 & 19.27 & 44.38 & 47.26 & 108.95 & 614.64 & 330.02 & 74.28\\
        \textbf{Avg Deg (G2)} & 8.01 & 8.15 & 10.56 & 26.33 & 43.83 & 66.13 & 206.03 & 153.62 & 51.49\\
        \textbf{Clustering} & 0.22 & 0.24 & 0.06 & 0.13 & 0.28 & 0.38 & 0.13 & 0.20 & 0.13 \\
        \textbf{Shortest Paths} & 3.47 & 6.31 & 6.34 & 6.06 & 5.38 & 4.76 & 4.23 & 5.58 & 5.92\\
        \textbf{Transitivity} & 0.22 & 0.09 & 0.05 & 0.04 & 0.17 & 0.12 & 0.01 & 0.02 & 0.03\\
        \midrule
        \textbf{Deg Gini} & 0.41 & 0.41 & 0.60 & 0.82 & 0.70 & 0.64 & 0.62 & 0.63 & 0.58\\
        \textbf{Coreness Gini} & 0.26 & 0.21 & 0.39 & 0.75 & 0.60 & 0.47 & 0.40 & 0.48 & 0.40\\
        \textbf{Heterogeneity} & -0.08 & 0.13 & 0.22 & 0.70 & 0.38 & 0.30 & 0.78 & 0.70 & 0.36\\
        \textbf{Power Law} $\alpha$ & 1.99 & 2.39 & 2.66 & 2.90 & 1.83 & 1.64 & 1.83 & 1.73 & 1.77\\
        \midrule
        \textbf{Edge Homophily} & 0.61 & 0.51 & 0.27 & 0.88 & 0.44 & 0.36 & 0.59 & 0.50 & 0.64\\
        \textbf{Jensen-Shannon} & 0.27 & 0.23 & 0.18 & 0.30 & 0.18 & 0.24 & 0.29 & 0.30 & 0.27\\
        \textbf{Mean Length} & 162 & 165 & 282 & 209 & 264 & 168 & 190 & 198 & 161\\
        \midrule
        \textbf{\textcolor{blue}{Min-Max Length}} & \textcolor{blue}{79--293} & \textcolor{blue}{8--929} & \textcolor{blue}{4--994} & \textcolor{blue}{7--1617} & \textcolor{blue}{42--329} & \textcolor{blue}{4--7281} & \textcolor{blue}{16--610} & \textcolor{blue}{20--2214} & \textcolor{blue}{4--774}\\
        \textbf{\textcolor{blue}{Data Points (M)}} & \textcolor{blue}{0.139} & \textcolor{blue}{2.57} & \textcolor{blue}{31.01} & \textcolor{blue}{57.26} & \textcolor{blue}{42.16} & \textcolor{blue}{37.44} & \textcolor{blue}{99.58} & \textcolor{blue}{190.27} & \textcolor{blue}{1034.18}\\
        \bottomrule
    \end{tabular}
    \end{adjustbox}

    \label{tab:graph-stats}

\end{table}

%

\subsection{Structure Statistics}
\label{subsec:app-dataset_statistics}
Considerable work has demonstrated that local and global structural characterizations are more effective for LP. To translate the homophily assumption, local and global graph heuristics, small-world phenomenon, and scale-free network properties into task-specific statistics, we provided the following graph metrics. Text length and diversity are critical factors influencing the performance of a model. In this context, we present the distribution of text lengths (\textbf{Min-Max Length} row) and per-dataset data volume (\textbf{Data Points} row) within our dataset in Table~\ref{tab:graph-stats}.

\begin{enumerate}[left=0pt, labelsep=0.5em, itemsep=0em]
    \item \textbf{Graph Density}: Number of Nodes, Edges, Arithmetic Deg are used to measure the graph's size, density and sparsity. Average degree of each central node $v \in \mathcal{V}$ and its of 2-order neighborhood $\mathcal{N}_{v}$' average degree measures the graph's local connectivity.
    \item \textbf{Graph Connectivity}: Clustering refers to fraction of possible triangles through one node exists:
    \begin{equation}
        C_i = \frac{2 \times T(i)}{\text{deg}(i)(\text{deg}(i) - 1)} \quad T = \frac{3 \times \text{\# triangles}}{\text{\# triads}}
    \end{equation}
    where $T(i)$ is the number of triangles through node $deg(i)$ and $deg(i)$ is the degree of node $i$. Transitivity measures the fraction of all possible triangles in the graph. 
    \item \textbf{Hierarchical level}: We leverage k-Core graph's fraction and degree distribution to calculate Gini and Coreness Gini. 
{\sloppy
\item \textbf{Scale-free}: If its node degree distribution $P(d)$ follows a power law $P(d) \sim d^{-\gamma}$, where $\gamma$ typically lies within the range $2 < \gamma < 3$. We approximate the power law $\alpha$ based on the following estimator. The datasets \texttt{pwc\_small}, \texttt{cora}, \texttt{pubmed}, \texttt{arxiv\_2023}, and \texttt{pwc\_large} are scale-free networks.\par}
    \begin{equation}
        \hat{\alpha} = 1 + N\left( \sum_{i=1}^{n} \log\left(\frac{d_i + 1}{d_{\min} + 1}\right) \right)^{-1}
    \end{equation}
\end{enumerate}

\subsection{Detailed Information about Datasets}
\label{subsec:detail-dataset}
In this work, we evaluate the performance of the models on the Link Prediction task using a diverse collection of reference datasets. Our analysis includes nine datasets obtained from three main sources: Planetoid \citep{Sen2008CollectiveClassification}, Amazon \citep{Shchur2018Pitfalls} and OGB \citep{Hu2020OGB}. The Planetoid datasets include \textit{Cora} and \textit{PubMed}. The Amazon datasets include \textit{Photo}. The OGB datasets include \textit{ogbn-arxiv}, \textit{Citationv8} and \textit{Arxiv\_2023}. And datasets without a group are \textit{Pwc\_small}, \textit{Pwc\_medium} and \textit{History}. The detailed statistics for each dataset are presented in Table~\ref{tab:graph-stats}. Below are descriptions of each dataset that were used in the experiments:
\begin{itemize}
    \item \textbf{Cora \citep{McCallum2000}}: It consists of 2,708 scientific publications, it contains 5,429 links and each paper is either cited or referenced by at least one other paper. Each publication in the dataset is described by a 0/1-value vector indicating the absence/presence of the corresponding word in the dictionary. The dictionary consists of 1,433 unique words.
    \item \textbf{Pubmed \citep{Sen2008CollectiveClassification}}: It contains 19,717 scientific publications from the PubMed database about diabetes research. It includes a citation network with 44,338 links, where nodes represent publications and edges denote citation relationships. Each publication is characterized by a TF/IDF weighted word vector derived from a dictionary of 500 unique terms.
    \item \textbf{Photo \citep{Shchur2018Pitfalls}}: The Amazon Photos dataset represents a collaborative shopping network, where nodes correspond to one of eight product categories, and edges indicate co-purchase relationships between products. Each node is characterized by a fixed-size object vector with 745 dimensions, which captures the relevant attributes of the corresponding product. 
    \item \textbf{Ogbn-arxiv \citep{Hu2020OGB}}: The ogbn-arxiv dataset is a directed graph representing a network of citations of computer science articles from arXiv indexed by MAG \citep{Wang2020MAG}. The nodes correspond to individual papers, each of which is described by a 128-dimensional vector of characteristics obtained from the embeddings of words in their titles and annotations, while the directed edges indicate the citation ratio.
    \item \textbf{Arxiv 2023 \citep{he2023harnessing}}: The Arxiv 2023 dataset is a citation network in computer science from arXiv, in particular those published in 2023 or later. The nodes correspond to individual papers, and the edges represent citation relationships. 
    \item \textbf{Pwc Medium \citep{Saier2023CoConAD}}: The Papers With Code (PWC) medium dataset is a comprehensive resource for exploring the use of research artifacts such as datasets, methods, models, and tasks. It includes rich meta-information, such as paper descriptions, categories, and links to the corresponding code and articles. In this work, we only utilize the title and abstract as node features. 
    \item \textbf{Pwc Small}: In order to fill the dataset scale gap in a hundred nodes, we extracted one connected component from PWC medium dataset as a small graph. 
    \item \textbf{History \citep{li2024tegdb}}: The Goodreads-History is a dataset in book recommendations. The Goodreads datasets are the main source. Nodes represent meta information of nodes such as types of books and reviewers, while edges indicate book reviews. Node labels are assigned based on the book categories. The descriptions of books and user information are used as book and user node textual information. The task is to predict the preference of users for products.
    \item \textbf{Citationv8 \citep{yan2023a}}:  The CitationV8 dataset is an extracted graph extracted from DBLP \citep{10.1145/1401890.1402008}. Node textual attributes in CitationV8 are derived from the titles and abstracts. Each edge represents a citation relationship between two papers.

\end{itemize}

\end{document}